%% file: main.tex
\documentclass[lettersize,journal]{IEEEtran}

\ifCLASSOPTIONcompsoc
  \usepackage[nocompress]{cite}
\else
  \usepackage{cite}
\fi

\usepackage{graphicx}
\usepackage{amsmath}
\usepackage{amssymb}
\usepackage{amsthm} 
\usepackage{booktabs}
\usepackage{xcolor}
\usepackage{enumitem}
\usepackage{microtype}
\usepackage{multirow}
\usepackage{makecell}
\usepackage[ruled, lined, linesnumbered, commentsnumbered, longend]{algorithm2e}
\usepackage{wrapfig}
\usepackage{soul}

\usepackage[pagebackref,breaklinks,colorlinks]{hyperref}

\usepackage[capitalize]{cleveref}
\crefname{section}{Sec.}{Secs.}
\Crefname{section}{Section}{Sections}
\Crefname{table}{Table}{Tables}
\crefname{table}{Tab.}{Tabs.}

\newcommand{\rot}{\textrm{rot}}

\newcommand{\pp}{\textrm{PP}}
\newcommand{\ppl}{\textrm{PPL}}

\newcommand{\symm}{\textrm{SPPL}}

\newcommand{\alignm}{\textrm{align}}
\newcommand{\smooth}{\textrm{smo}}
\newcommand{\arap}{\textrm{ARAP}}
\newcommand{\proj}{\textrm{proj}}

\newcommand{\corresidx}[1]{\rho_{#1}}
\newcommand{\prevcorresidx}[1]{\rho'_{#1}}
\newcommand{\newpos}[1]{\widehat{#1}}
\newcommand{\pointpairweight}[1]{\alpha_{#1}}
\newcommand{\corresidxiter}[2]{\rho_{#1}^{(#2)}}
\newcommand{\sparename}{Symmetrized Point-to-plAne distance for Robust non-rigid 3D rEgistration}
\DeclareMathOperator{\diag}{diag}
\newcommand{\mypara}[1]{\vspace{0.5em}\noindent\textbf{#1.}~}

\newcommand{\surfunc}{\overline{F}}
\newcommand{\origfunc}{F}
\newcommand{\leftsvdmat}{\mathbf{U}_\mathbf{S}}
\newcommand{\rightsvdmat}{\mathbf{V}_\mathbf{S}}
\newcommand{\correrr}{\text{Corr}_\text{err}}

\DeclareMathOperator*{\argmin}{arg\,min}

\newtheorem{proposition}{Proposition}
\usepackage{ulem}
\newcommand{\red}[1]{\textcolor{red}{#1}}

\begin{document}

\title{SPARE: Symmetrized Point-to-Plane Distance for Robust Non-Rigid 3D Registration}

\author{
	Yuxin Yao,~
	Bailin Deng,~\IEEEmembership{Member,~IEEE},~
	Junhui Hou,~\IEEEmembership{Senior Member,~IEEE,}~
	Juyong Zhang,~\IEEEmembership{Member,~IEEE}	
	\IEEEcompsocitemizethanks{
    This research was in part supported by the National Natural Science Foundation of China(No.62441224, No.62272433, No.62422118), the Hong Kong Research Grants Council under Grants 11219324 and 11219422, and the USTC Fellowship (No.S19582024). (\textit{Corresponding author: Juyong Zhang.})
    
    \IEEEcompsocthanksitem 
    {Y. Yao is with the School of Mathematical Sciences, University of Science and Technology of China,  and also with the Department of Computer Science, City University of Hong Kong. Email: yaoyuxin@mail.ustc.edu.cn.}
		\IEEEcompsocthanksitem B. Deng is with the School of Computer Science and Informatics, Cardiff University. Email: DengB3@cardiff.ac.uk.
		\IEEEcompsocthanksitem J. Hou is with the Department of Computer Science, City University of Hong Kong. Email: jh.hou@cityu.edu.hk.
        \IEEEcompsocthanksitem J. Zhang is with the School of Mathematical Sciences, University of Science and Technology of China. Email: juyong@ustc.edu.cn. 
  }
}

\markboth{}%
{}



\IEEEtitleabstractindextext{
\begin{abstract}
Existing optimization-based methods for non-rigid registration typically minimize an alignment error metric based on the point-to-point or point-to-plane distance between corresponding point pairs on the source surface and target surface. However, these metrics can result in slow convergence or a loss of detail.
In this paper, we propose \emph{SPARE}, a novel formulation that utilizes a symmetrized point-to-plane distance for robust non-rigid registration. The symmetrized point-to-plane distance relies on both the positions and normals of the corresponding points, resulting in a more accurate approximation of the underlying geometry and can achieve higher accuracy than existing methods. 
To solve this optimization problem efficiently, we introduce an as-rigid-as-possible regulation term to estimate the deformed normals and propose an alternating minimization solver using a majorization-minimization strategy. 
Moreover, for effective initialization of the solver, we incorporate a deformation graph-based coarse alignment that improves registration quality and efficiency. 
Extensive experiments show that the proposed method greatly improves the accuracy of non-rigid registration problems and maintains relatively high solution efficiency.
The code is publicly available at \url{https://github.com/yaoyx689/spare}.
\end{abstract}

}

\maketitle

\IEEEdisplaynontitleabstractindextext

\IEEEpeerreviewmaketitle

\input{intro}
\input{related}
\input{method}
\input{experiment}
\input{conclusion}

\bibliographystyle{IEEEtran}
\bibliography{egbib}

\begin{IEEEbiography}[{\includegraphics[width=1in]{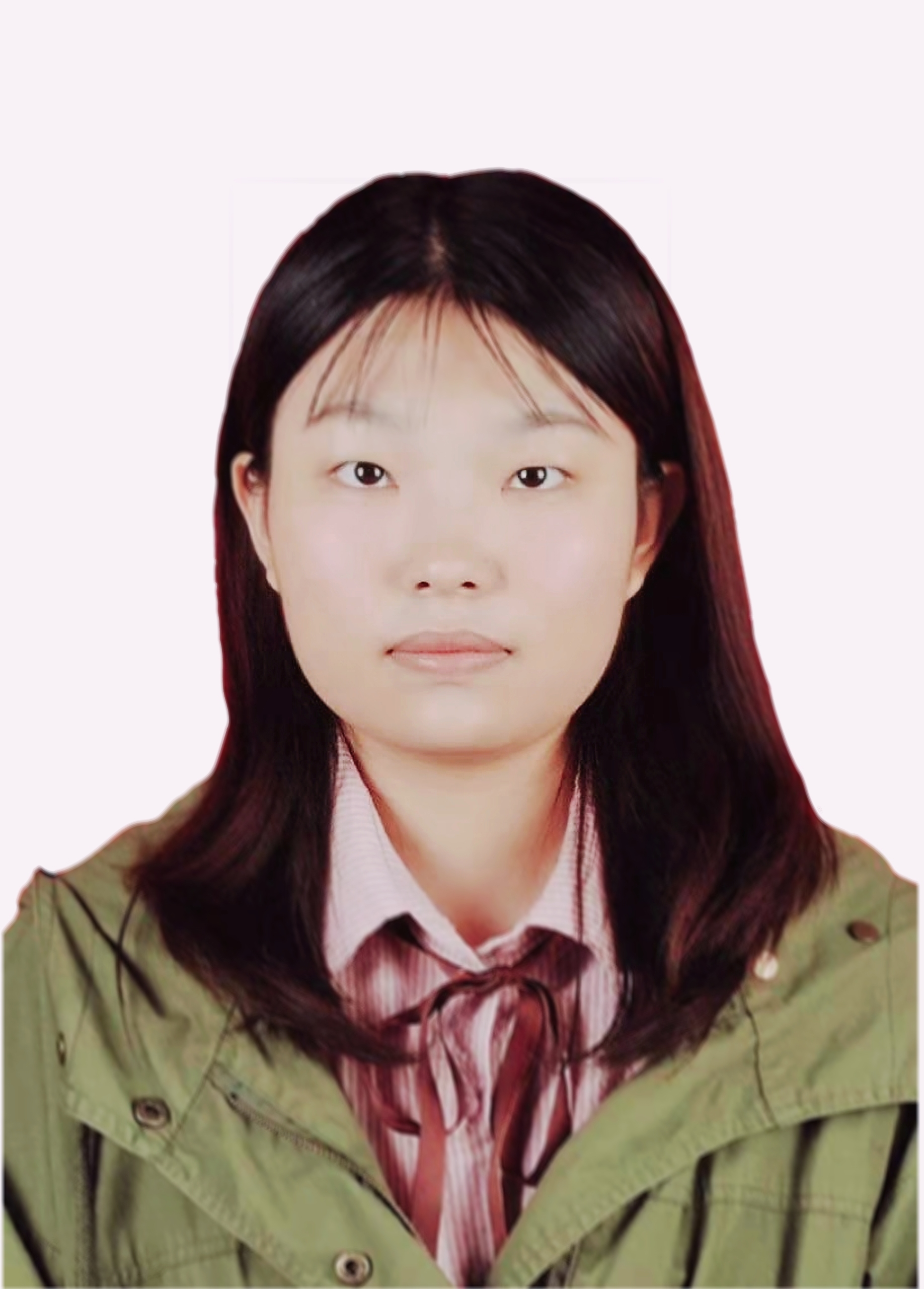}}]{Yuxin Yao} is currently a post-doctoral
researcher with the Department of Computer Science, City University of Hong Kong. She received the BE degree from the University of Electronic Science and Technology of China, Chengdu,
China, in 2018, and the PhD degree from the University of Science and Technology of China, Hefei, China, in 2023. Her research interests include computer vision, computer graphics, and 3D reconstruction.
\end{IEEEbiography}

\begin{IEEEbiography}[{\includegraphics[width=1in]{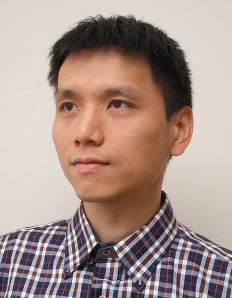}}]{Bailin Deng}
	is a Senior Lecturer in the School of Computer Science and Informatics at Cardiff University. He received the BEng degree in computer software (2005) and the MSc degree in computer science (2008) from Tsinghua University (China), and the PhD degree in technical mathematics (2011) from Vienna University of Technology (Austria). His research interests include geometry processing, numerical optimization, computational design, and digital fabrication. He is a member of IEEE, and an associate editor of IEEE Computer Graphics and Applications.
\end{IEEEbiography}

\begin{IEEEbiography}[{\includegraphics[width=1in]{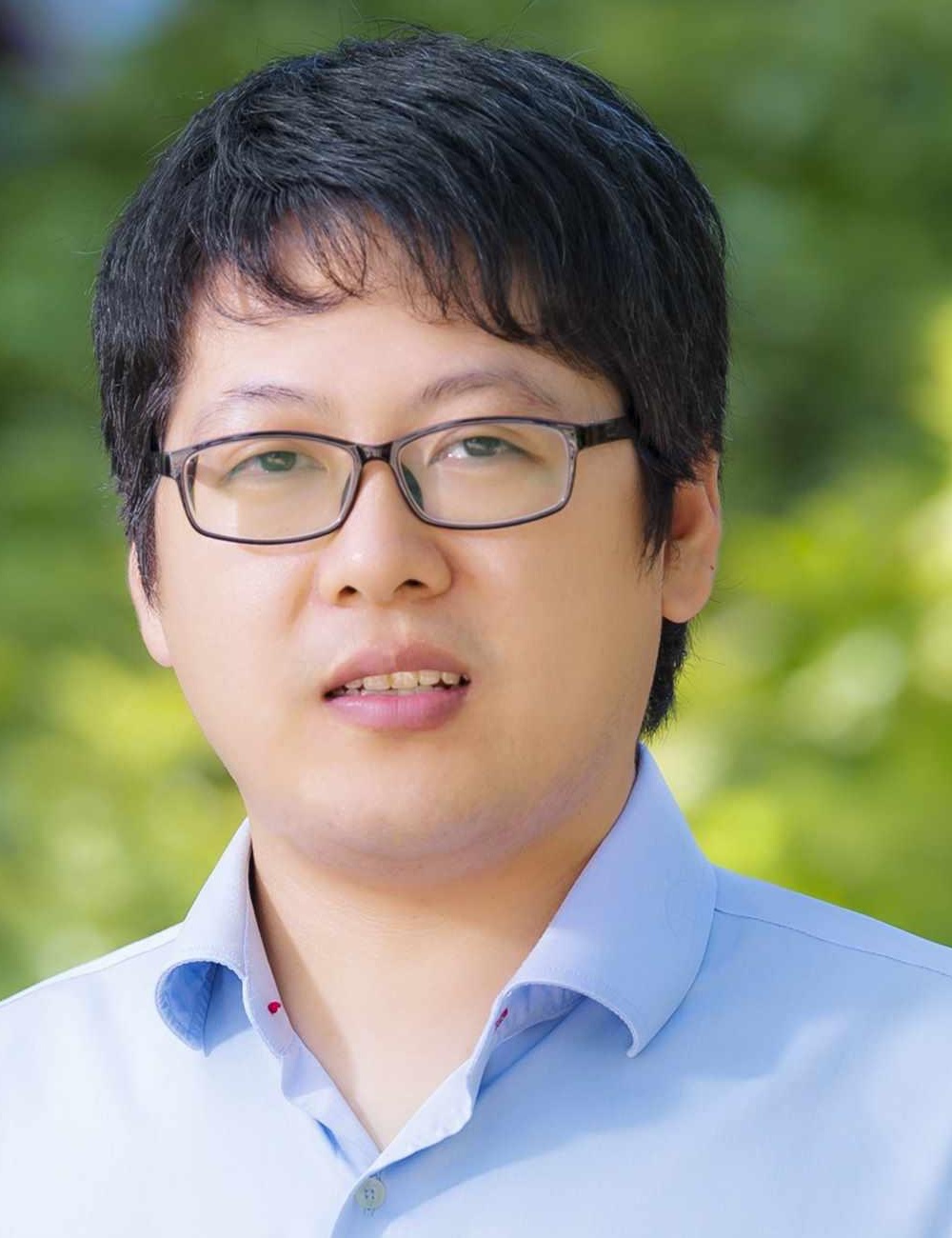}}]{Junhui Hou} is an Associate Professor with the Department of Computer Science, City University of Hong Kong. He holds a B.Eng. degree in information engineering (Talented Students Program) from the South China University of Technology, Guangzhou, China (2009), an M.Eng. degree in signal and information processing from Northwestern Polytechnical University, Xi’an, China (2012), and a Ph.D. degree from the School of Electrical and Electronic Engineering, Nanyang Technological University, Singapore (2016). His research interests are multi-dimensional visual computing.

Dr. Hou received the Early Career Award (3/381) from the Hong Kong Research Grants Council in 2018 and the NSFC Excellent Young Scientists Fund in 2024. He has served or is serving as an Associate Editor for \textit{IEEE Transactions on Visualization and Computer Graphics}, \textit{IEEE Transactions on Image Processing}, \textit{IEEE Transactions on Multimedia}, and \textit{IEEE Transactions on Circuits and Systems for Video Technology}. 
\end{IEEEbiography}

\begin{IEEEbiography}[{\includegraphics[width=1in]{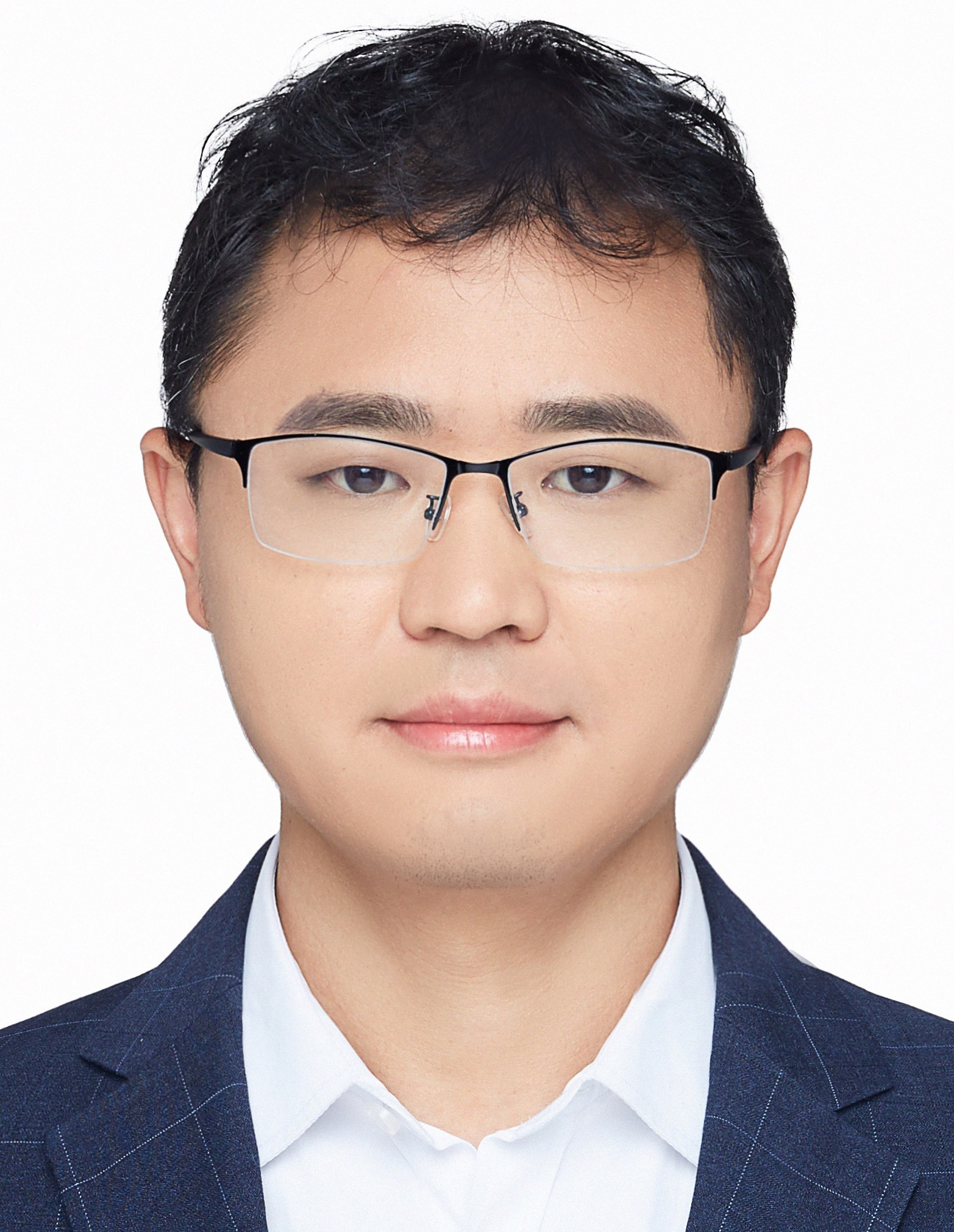}}]{Juyong Zhang}
is a professor in the School of Mathematical Sciences at University of Science and Technology of China. He received the BS degree from the University of Science and Technology of China in 2006, and the PhD degree from Nanyang Technological University, Singapore. His research interests include computer graphics, computer vision, and numerical optimization. He is an associate editor of IEEE Transactions on Mobile Computing and IEEE Computer Graphics and Applications.

\end{IEEEbiography}

\vfill

\clearpage
\input{supplementary}

\end{document}

%% file: intro.tex
\section{Introduction}
\IEEEPARstart{G}{iven} a source surface and a target surface, non-rigid registration aims to compute a deformation field that aligns the source surface with the target surface. This problem is fundamental in computer vision, with various applications such as 3D shape acquisition and tracking. 

In non-rigid registration, the challenge lies in aligning two surfaces without knowing their correspondence in advance.
Motivated by the ICP algorithm~\cite{Besl1992} for rigid registration, many non-rigid registration methods iteratively update the correspondence using closest point queries from the source points to the target surface, and then update the deformation by minimizing an alignment error metric between the corresponding points along with some regularization terms for the deformation field~\cite{Amberg2007,li2019robust,yao2022fast}. 
The quality and speed of the registration are heavily influenced by the alignment error metric.
A commonly used metric is the ``point-to-point'' distance, measured between points on the source surface and their corresponding closest points on the target surface~\cite{allen2003space,Amberg2007,li2008global,chang2011global,li2019robust,Yao_2020_CVPR,yao2022fast}.  
However, due to the unreliability of such correspondence, the registration may converge to a local minimum and produce unsatisfactory results. 
Another frequently used metric is the ``point-to-plane'' distance, which measures the distance between source surface points and the tangent planes at their closest points on the target surface~\cite{li2009robust, xu2017flycap, dou2016fusion4d, yu2018doublefusion, li2020robust2}. 
As the point-to-plane distance is a first-order approximation of the target surface shape around the closest point, it provides a more accurate proxy of the target shape than the point-to-point distance and enables faster alignment of the two surfaces.
Despite these efforts, non-rigid registration can still be challenging and time-consuming. In the rigid registration literature, alignment error metrics that incorporate higher-order geometric properties have been utilized to achieve faster convergence than the point-to-plane distance~\cite{Mitra2004,pottmann2006geometry}, but such higher-order properties are expensive to compute. Recently, a symmetrized point-to-plane metric was proposed  in~\cite{Rusinkiewicz2019} for rigid registration. It measures the consistency between the positions and normals of each source point and its corresponding closest point, and achieves the minimal value when the pair of points lies on a second-order patch of surface~\cite{Rusinkiewicz2019}. When used for rigid registration, this gains similar benefits of fast convergence as registration methods that exploit second-order surface properties, but without the need for explicit computation of such properties. 

In this paper, we propose \emph{SPARE}, a novel method that utilizes the \sparename{}. Although the symmetrized point-to-plane distance works well on rigid registration, deploying it for non-rigid registration is a non-trivial task. First, compared to rigid transformations, accurately modeling the point normals on surfaces undergoing complex non-rigid deformations presents significant computational challenges due to their nonlinear geometric dependencies. Second, there is a large space of point pair positions and normals that can minimize the symmetrized point-to-plane distance. While this is beneficial for rigid registration problems with a limited degree of freedom, it can be problematic for non-rigid registration: due to the excessive degrees of freedom, the symmetrized point-to-plane distance may be minimized without actually aligning the two surfaces. 
Finally, in real-world applications, non-rigid registration is often carried out between surface data with noise, outliers or partial overlaps, which may lead to erroneous correspondence for closest-point queries and impact the effectiveness of the symmetrized point-to-plane distance.  
To estimate point normals after deformation, we assume local surface deformation is nearly rigid. We therefore introduce an as-rigid-as-possible (ARAP) regularization term. This ARAP term, in conjunction with the symmetrized point-to-plane distance, contributes to a more accurate approximation of the distance between surfaces. Furthermore, to address the problem of excessive degrees of freedom, we incorporate a deformation graph-based coarse alignment. This not only helps the deformation field to maintain local structures of the source shape, but also reduces the degrees of freedom and restores the effectiveness of the symmetrized point-to-plane distance for non-rigid registration. 
In addition, to address the issue of erroneous correspondence due to noise, outliers and partial overlaps, we incorporate an adaptive weight into the symmetrized point-to-plane metric, which is computed according to the positions and normals of each point pair and automatically reduces the influence of unreliable correspondence.

As the resulting optimization problem is non-convex and highly non-linear, we devise an iterative solver that alternately updates the subsets of the variables, decomposing the optimization into simple sub-problems that are easy to solve. In particular, in each iteration we adopt a majorization-minimization strategy~\cite{lange2016mm} and construct a proxy problem that minimizes a simple surrogate function that bounds the target function from above, enabling us to derive closed-form update steps that can effectively reduce the original target function.
Moreover, to effectively initialize the solver and avoid undesirable local minima, we propose a strategy that uses a deformation graph~\cite{sumner2007embedded} to roughly align the two surfaces while maintaining the structure of the source shape.

We perform extensive experiments to evaluate our method on both synthetic and real datasets. Our method outperforms state-of-the-art optimization-based and learning-based methods in terms of accuracy and efficiency.
To summarize, our contributions include:
\begin{itemize}[leftmargin=*]
    \item We formulate a novel optimization-based approach for non-rigid registration problem. Our formulation adopts a robustified symmetrized point-to-plane distance in conjunction with an as-rigid-as-possible regularization term, significantly enhancing the robustness and accuracy of the registration solution.
    \item We propose an efficient alternating minimization solver for the resulting non-convex optimization problem, using a majorization-minimization strategy to derive simple sub-problems that effectively reduce the target function.
    \item We devise an effective initialization strategy for the solver, by using a deformation graph-based coarse alignment that maintains the overall structure of the source shape. 
\end{itemize}

%% file: related.tex
\section{Related work}

In the following, we review existing works closely related to our approach. More complete overviews of non-rigid registration can be found in recent surveys such as~\cite{Sahillioglu2020,Deng2022}.

\subsection{Alignment Error Metric}
 For optimization-based non-rigid registration, the alignment error metric is an important component in the formulation. Many existing methods adopt the simple point-to-point distance from the classical ICP algorithm~\cite{Amberg2007,li2008global,liao2009modeling,hontani2012robust,Yao_2020_CVPR}. Others use the point-to-plane distance instead~\cite{dou2016fusion4d, yu2017bodyfusion,dou2017motion2fusion,li2018articulatedfusion,yu2018doublefusion,zhang2019interactionfusion,li2020robust2}, which benefits from faster convergence than the point-to-point distance albeit being more complicated to solve.
 Some methods utilize both metrics to complement each other~\cite{li2009robust,chang2011global, guo2015robust,  thomas2016augmented,li2020topology,su2020robustfusion,zampogiannis2019topology}.
 Another class of methods, such as the coherent point drift (CPD)~\cite{myronenko2010point}, models the sample points using Gaussian mixtures and formulates the alignment error from a probabilistic perspective~\cite{myronenko2010point,Jian2011robust,ma2013robust,ge2014non}. Recently, a Bayesian formulation of CPD has been proposed in~\cite{hirose2021bayesian}, with further work in~\cite{hirose2020acceleration,Hirose2022geodesic} to improve its performance and accuracy. In~\cite{fan2022coherent}, CPD is also generalized to non-Euclidean domains to improve its robustness on data with large deformations. {However, when the data quality is poor, e.g., with the presence of noise or partial overlap, the registration quality will degrade rapidly.}

{To address this problem,} some methods {assign} individual weights to each point pair in the alignment metric to reduce the impact of erroneous correspondence~\cite{Amberg2007, li2008global, chang2011global}. Others apply a robust norm (such as the $\ell_0$-norm or the Welsch's function) to the alignment metric, which automatically reduces the contribution of point pairs that are less reliable~\cite{newcombe2015dynamicfusion, yu2018doublefusion, xu2019unstructuredfusion, su2020robustfusion, Yao_2020_CVPR}. {Recently, DDM~\cite{ren2025ddm} proposes a distance metric based on directional distance fields (DDF) featuring a confidence score for enhanced alignment of 3D geometric surfaces. Separately, \cite{Rusinkiewicz2019} introduces a symmetrized point-to-plane metric that considers higher-order properties of surfaces. Our method incorporates this latter metric from~\cite{Rusinkiewicz2019} with an adaptive weighting scheme for non-rigid registration; this approach yields a robust alignment error metric and ultimately leads to improved alignment performance.}

 \subsection{{Deformation Field}}
 The performance of non-rigid registration also depends on the representation of the deformation field. A simple approach is to introduce a variable for the new position or transformation of each source point~\cite{liao2009modeling, huang2011global,li2019robust}. Such a representation provides abundant degrees of freedom to represent complex deformations, but the resulting problem is often expensive to solve due to the large number of variables.
 Other methods address this problem using an embedded deformation graph~\cite{sumner2007embedded}, where each source point is deformed according to the transformations associated with the nearby nodes of a small graph on the surface~\cite{li2008global, li2009robust, li2017robust, Yao_2020_CVPR}. This reduces the number of variables and improves computational efficiency, at the cost of less expressiveness due to fewer degrees of freedom. 
 \cite{bogo2014faust,bogo2017dfaust} leverage model priors to parameterize deformations into a set of pose and shape parameters.
 Recent advancements in neural networks and deep learning have introduced new strategies for representing deformation fields. Some methods directly leverage neural networks for this task~\cite{eisenberger2021neuromorph,prokudin2023dynamic}. Others focus on reduced-dimensional representations, employing deformation nodes, control points, or neural bones to represent the reduced-dimensional deformation field~\cite{sundararaman2022reduced,yang2022banmo,yao2024dynosurf}.
 Various specific techniques have also emerged: for instance,
 \cite{corona2022learned,marin2024nicp} iteratively predict the offset of each point based on a per-vertex neural network; 
 \cite{prokudin2019efficient} introduces a basis point set representation to predict the positions of deformed mesh points; \cite{feng2021recurrent} utilizes a combination of rigid transformations to predict non-rigid deformations. \cite{lei2022cadex,Cai2022NDR} adopt a bijective map to model the reversible deformation field between two shapes; and \cite{groueix20183d} employs the structure of an encoder and decoder to predict deformation.
 However, these learning-based approaches can exhibit data dependencies. In this paper, we propose a combination of coarse alignment using a deformation graph and fine registration based on pointwise variables, designed to achieve high solution accuracy while maintaining fast solution speed.

\subsection{{Correspondences}}
Due to initially unknown correspondences, optimization-based methods generally follow the ICP algorithm's paradigm~\cite{Besl1992}, which involves iteratively regarding the closest points as correspondences and applying optimization steps to gradually converge to an optimal solution.
In recent years, deep learning techniques have also been adopted to deal with challenging non-rigid registration problems. 
Many methods rely on supervised learning to improve the quality of the alignment.
Some methods learn the correspondences between source and target surfaces from ground-truth correspondences, which helps to handle data with large deformations~\cite{wei2016dense, bozic2020deepdeform,lepard2021,trappolini2021shape,qin2023deep,bhatnagar2020loopreg,wang2019deep}.
Others use the ground-truth positions of the deformed model as the supervision to train the network to learn the correspondences or the deformation fields~\cite{wang2020sequential, bozic2020deepdeform, bovzivc2020neural}.
Despite their strong capabilities, these methods rely on training data with ground-truth correspondence or deformation, which is not always easy to obtain.
Some unsupervised methods that do not require ground-truth correspondence have also been proposed to address this problem~\cite{feng2021recurrent, zeng2021corrnet3d,Wang20193DN,marin2024nicp}. 
Such methods still rely on a properly designed loss function that is minimized during training to help the network learn the registration. 
Some methods~\cite{bozic2020deepdeform,bovzivc2020neural} incorporate additional information, such as texture or colors, to assist in the establishment of corresponding points.
Despite the rapid progress in learning-based methods, research into optimization approaches still advances the field in a complementary way: such approaches do not rely on data and 
have stronger generalization capability. On the data with relatively small deformation, they can usually achieve higher accuracy. They may be used to compute the required ground-truth information for supervised learning, and they can also contribute to the knowledge of loss function design for unsupervised methods. Furthermore, optimization-based methods can be integrated with correspondences predicted by learning-based methods to reduce reliance on the spatial positions of the source and target, thereby achieving better performance. 

%% file: method.tex
\section{Proposed Method}

Consider a source surface represented by sample points $\mathcal{V}=\{\mathbf{v}_1, ..., \mathbf{v}_{|\mathcal{V}|} \in \mathbb{R}^3 \}$ equipped with normals $\mathcal{N}_s=\{\mathbf{n}_1, ..., \mathbf{n}_{|\mathcal{V}|} \in \mathbb{R}^3\}$, and a target surface represented by sample points $\mathcal{U}=\{\mathbf{u}_1, ..., \mathbf{u}_{|\mathcal{U}|} \in \mathbb{R}^3\}$ with normals $\mathcal{N}_t=\{\mathbf{n}_1^t, ..., \mathbf{n}_{|\mathcal{U}|}^t \in \mathbb{R}^3\}$.
We assume that the neighboring relation between the source sample points in $\mathcal{V}$ is represented by a set of edges $\mathcal{E} = \{(\mathbf{v}_i, \mathbf{v}_j)\}$.
This representation is applicable to both meshes and point clouds: the sample points are either mesh vertices or points within a point cloud, while the edges are either mesh edges or derived from the $k$-nearest neighbors in the point cloud.
The normals can come directly from the input data, be estimated from nearby points using PCA, or calculated by averaging the normals of adjacent faces in a mesh.

We aim to compute a deformation field for the source surface to align it with the target surface. In practical applications, the non-rigid registration problem is inherently challenging due to partial overlap, the lack of known corresponding point pairs, and the presence of noise and outliers in both the source and target surfaces. 
To this end, we propose a novel optimization formulation for non-rigid registration, utilizing a robust symmetrized point-to-plane (SP2P) distance metric and an as-rigid-as-possible (ARAP) regularization term in the target function (Sections~\ref{Sec:symmDistance} \& \ref{Sec:alignment}). To efficiently solve the resulting non-convex optimization, we derive a majorization-minimization (MM) solver that decomposes it into simple sub-problems with closed-form solutions (Section~\ref{Sec:numericalSolver}). Furthermore, to avoid local minima and improve the solution accuracy and speed, we introduce a deformation graph-based coarse alignment to initialize our MM solver (Section~\ref{sec:coarseAlignment}).

\subsection{Preliminary: Symmetrized Point-To-Plane Distance}
\label{Sec:symmDistance}

\begin{figure*}[!htb]
  \centering
  \vspace{-2em}
\includegraphics[width=\textwidth]{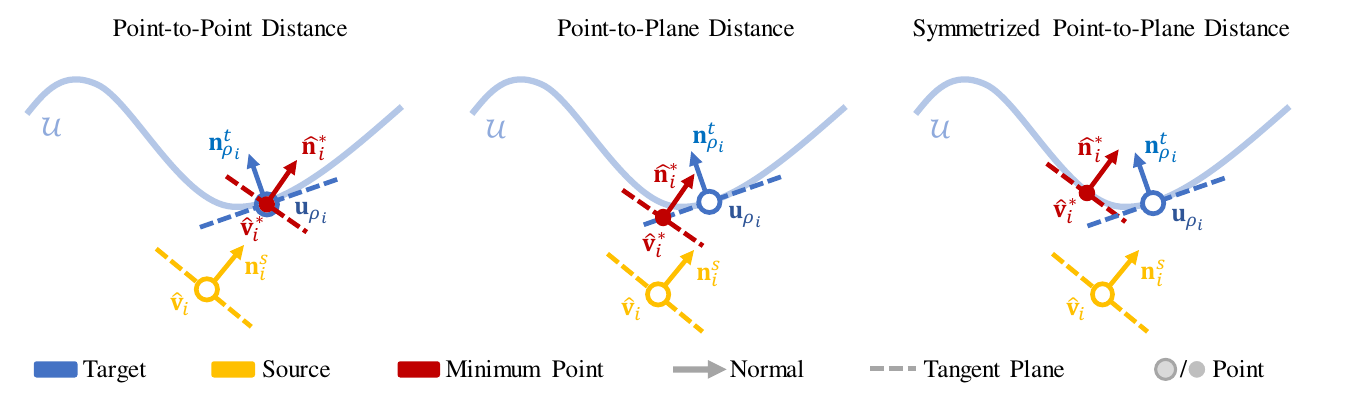} 
\vspace{-2.5em}
  \caption{The target positions of the points on the source surface when minimizing different distance metrics. Here, we only show the position and the normal vector of a potential minimum point. 
  When considering the point-to-plane distance, the position of the minimum point may slide along the tangent plane of the target point. 
  For the symmetrized point-to-plane distance, the minimum point may slide along a symmetric surface that passes through the source point. Additionally, for the point-to-point distance and the point-to-plane distance, the normal vector of the minimum point may rotate around the minimum point.} 
  \label{fig:min-distance} 
\end{figure*}

A central component for optimization-based registration is the alignment error metric that penalizes the deviation between the source surface and the target surface. Many existing non-rigid registration methods adopt a simple point-to-point metric that originates from rigid registration~\cite{Besl1992}:  
\begin{equation}
E_{\pp}^i =  \|\newpos{\mathbf{v}}_i - \mathbf{u}_{\corresidx{i}}\|^2,
\label{eq:P2P}
\end{equation}
where $\newpos{\mathbf{v}}_i$ is the new position of ${\mathbf{v}}_i$ after the deformation, and $\mathbf{u}_{\corresidx{i}}$ is the closest target point to  $\newpos{\mathbf{v}}_i$. This allows for a simple solver that iteratively updates the closest point correspondence and re-computes the deformation according to the correspondence~\cite{Amberg2007}, similar to the classical ICP algorithm for rigid registration~\cite{Besl1992}. 
However, this metric may not accurately measure the alignment error: as the source and target point clouds may be sampled from different locations of the underlying surface, they may not fully align under the ground-truth deformation, and the point-to-point metric may not reach its minimum in this case.
Moreover, the fixed closest points in each iteration fail to account for the change of correspondence for the moving source points, which can lead to slow convergence, especially when the alignment requires tangential movement along the target surface~\cite{pottmann2006geometry}.
To address such issues, other methods adopt a point-to-plane metric from the rigid registration literature~\cite{Chen1992}:
\begin{equation}
    E_{\ppl}^i =   [\mathbf{n}_{\corresidx{i}}^t\cdot(\newpos{\mathbf{v}}_i - \mathbf{u}_{\corresidx{i}})]^2,
    \label{eq:PPL}
\end{equation}
where $\mathbf{n}_{\corresidx{i}}^t$ is the normal vector at $\mathbf{u}_{\corresidx{i}}$.
The metric penalizes the deviation from the tangent planes at the closest target points.
Since the tangent plane is a first-order approximation of the target surface around the closest point, this metric helps to account for the target surface shape away from the sample points and leads to faster convergence than the point-to-point metric~\cite{pottmann2006geometry}. 
In the rigid registration literature, alignment error metrics that encode higher-order geometric properties (such as local curvatures) have also been adopted for even faster convergence than the point-to-plane metric~\cite{Mitra2004,pottmann2006geometry}. However, such higher-order properties can be expensive to compute.  

To achieve faster convergence than the point-to-plane distance without evaluating higher-order properties, a symmetrized point-to-plane (SP2P) metric (see Fig.~\ref{fig:min-distance}) has been proposed recently in~\cite{Rusinkiewicz2019} for rigid registration:
\begin{equation}
E_{\symm}^i = [(\newpos{\mathbf{n}}_i + \mathbf{n}_{\corresidx{i}}^t) \cdot (\newpos{\mathbf{v}}_i - \mathbf{u}_{\corresidx{i}})]^2,
\label{eq:symm}
\end{equation}
where $\newpos{\mathbf{n}}_i$ is the normal at the point $\newpos{\mathbf{v}}_i$ on the deformed  surface.
Here, the term $(\newpos{\mathbf{n}}_i + \mathbf{n}_{\corresidx{i}}^t) \cdot (\newpos{\mathbf{v}}_i - \mathbf{u}_{\corresidx{i}})$ vanishes when the corresponding
point pair $(\newpos{\mathbf{v}}_i, \mathbf{u}_{\corresidx{i}})$ and their normals $(\newpos{\mathbf{n}}_i, \mathbf{n}_{\corresidx{i}}^t)$ are consistent 
with a locally quadratic surface centered between them~\cite{Rusinkiewicz2019}.
In other words, the metric can be considered as penalizing the deviation between the deformed source point $\newpos{\mathbf{v}}_i$ and a family of quadratic approximations of the target surface that are consistent with the corresponding target point $\mathbf{u}_{\corresidx{i}}$  and its normal $\mathbf{n}_{\corresidx{i}}^t$. It has been observed in~\cite{Rusinkiewicz2019} that the symmetrized point-to-plane metric~\eqref{eq:symm} leads to faster convergence of ICP than the point-to-plane metric~\eqref{eq:PPL}.

\subsection{Non-rigid Registration with Robust SP2P Distance}
\label{Sec:alignment}

In this paper, we adopt the SP2P distance in Eq.~\eqref{eq:symm} to propose a new optimization formulation for non-rigid registration, to benefit from its fast convergence. However,  in real-world applications, the closest-point correspondence can become unreliable due to partial overlap, noise, outliers, and large initial position differences in the input surfaces.
In such cases, the error metric on unreliable point pairs may steer the optimization toward an erroneous alignment.
One common solution is to use dynamically adjusted weights for individual source points to control their contribution to the alignment energy based on the reliability of their correspondence~\cite{Rusinkiewicz2001}. Following this idea, we introduce a weighted alignment term for the source point ${\mathbf{v}}_i$ as
\begin{equation}
    E_{\alignm}^i = \pointpairweight{i}[(\newpos{\mathbf{n}}_i + \mathbf{n}_{\corresidx{i}}^t) \cdot (\newpos{\mathbf{v}}_i - \mathbf{u}_{\corresidx{i}})]^2,
    \label{eq:alignment}
\end{equation}
where the weight $\alpha_i$ is computed based on the deformed position $\newpos{\mathbf{v}}'_i$, the corresponding point $\mathbf{u}_{\prevcorresidx{i}}$, and their normals $\newpos{\mathbf{n}}'_i, \mathbf{n}^t_{\prevcorresidx{i}}$ in the previous iteration:
\begin{equation}
\label{eq:robustweight}
\pointpairweight{i} = \left\{
\begin{aligned}
&0, &\newpos{\mathbf{n}}'_i\cdot  \mathbf{n}^t_{\prevcorresidx{i}}<0,\\
&\exp(-\frac{\|\newpos{\mathbf{v}}'_i-\mathbf{u}_{\prevcorresidx{i}}\|^2}{2\sigma^2}), & \text{otherwise}, 
\end{aligned}
\right.
\end{equation}
where $\sigma$ is a user-specified parameter. We set $\sigma$ to the median Euclidean distance from the initial source points to their closest target points.
The weight $\pointpairweight{i}$ has a large value only if both the positions and normals of the point pair are close to each other, thus reducing the influence of unreliable correspondence.

Another challenge in using the SP2P metric for non-rigid registration is that although it works well for rigid registration, using it alone for non-rigid registration is usually insufficient to achieve good results. This is because compared to the point-to-point and point-to-plane metrics, the SP2P metric has a much larger space of position and normal values with a zero metric value: on the zero level set, the positions and normals of the point pair are only required to be consistent with a certain quadratic surface patch, but the quadratic patch is not necessarily consistent with the target surface.
This is not an issue for rigid registration: since all the source points must undergo the same rigid transformation, this induces an implicit regulation for the deformed source points and their normals, forcing them to align with the target surface eventually. 
Furthermore, in non-rigid registration, the source surface may undergo different deformations, making it challenging to accurately model the normal direction after deformation. We notice that in many non-rigid registration problems such as human body tracking, the surface exhibits local rigidity during deformation. Therefore, we introduce for each source point an as-rigid-as-possible (ARAP) term~\cite{Sorkine2007arap} to help estimate the normal direction after deformation. The ARAP term constrains the deformation in its local neighborhood to be close to a rigid transformation:
\begin{equation}
E_{\arap}^i = \frac{1}{|\mathcal{N}(\mathbf{v}_i)|} \sum_{j\in \mathcal{N}(\mathbf{v}_i)} \|(\newpos{\mathbf{v}}_i-\newpos{\mathbf{v}}_j) - \mathbf{R}_i(\mathbf{v}_i-\mathbf{v}_j)\|^2.
\label{eq:arap-each-energy}
\end{equation}
Here $\mathcal{N}(\mathbf{v}_i)$ is the index set of neighboring points for $\mathbf{v}_i$, 
and $\mathbf{R}_i \in \mathbb{R}^{3 \times 3}$ is an auxiliary rotation matrix variable for a rigid transformation associated with the source point $\mathbf{v}_i$  (i.e., it needs to satisfy $\mathbf{R}_i^T \mathbf{R}_i = \mathbf{I}$  and $\det(\mathbf{R}_i) = 1$). Therefore, we can use the auxiliary rotation matrix $\mathbf{R}_i$ to approximate the normal of $\newpos{\mathbf{v}}_i$ after the deformation as 
\begin{equation}
\label{eq:updateNormal}
\newpos{\mathbf{n}}_i = \mathbf{R}_i\mathbf{n}_i,
\end{equation}
so that the alignment term in Eq.~\eqref{eq:alignment} becomes
\begin{equation}
    E_{\alignm}^i = \pointpairweight{i}[(\mathbf{R}_i\mathbf{n}_i + \mathbf{n}_{\corresidx{i}}^t) \cdot (\newpos{\mathbf{v}}_i - \mathbf{u}_{\corresidx{i}})]^2.
    \label{eq:alignmentwithR}
\end{equation}
Combining the alignment term~\eqref{eq:alignmentwithR} and the ARAP term~\eqref{eq:arap-each-energy} for all source points, we obtain the overall alignment term and ARAP term as
\begin{equation}
E_{\alignm} = \frac{1}{|\mathcal{V}|}\sum_{\mathbf{v}_i\in\mathcal{V}} E^i_{\alignm}, \quad E_{\arap} = \frac{1}{2|\mathcal{E}|}\sum_{\mathbf{v}_i\in\mathcal{V}} E^i_{\arap}.
\label{eq:OverallTerms}
\end{equation}
Our non-rigid registration is formulated as an optimization problem
\begin{equation}
\begin{aligned}
\min_{\{\newpos{\mathbf{v}}_i\}, \{\mathbf{R}_i\}} ~~& E_{\alignm} + w_{\arap} E_{\arap}\\
\text{s.t.} ~~& \mathbf{R}_i^T \mathbf{R}_i = \mathbf{I},~\det(\mathbf{R}_i) = 1,~~\forall~i,
\end{aligned}
\label{eq:FineOptimization}
\end{equation}
where $w_{\arap}$ is a user-specified parameter.

\subsection{Numerical Solver}
\label{Sec:numericalSolver}

The optimization problem~\eqref{eq:FineOptimization} is challenging due to the non-linearity, non-convexity, and constraints of rotation matrices. To solve it efficiently and effectively, we devise an iterative solver that alternately updates the closest points $\{\mathbf{u}_{\corresidx{i}}\}$, the position variables $\{\newpos{\mathbf{v}}_i\}$, and the rotation matrix variables $\{\mathbf{R}_i\}$ while fixing the remaining variables. We denote their values after the $k$-th iteration as $\{\mathbf{u}_{\corresidx{i}^{(k)}}\}$, $\{\newpos{\mathbf{v}}_i^{(k)}\}$, and $\{\mathbf{R}_i^{(k)}\}$, respectively. In the following, we explain how to update their values in the ($k$+1)-th iteration.

\mypara{Update of $\{\mathbf{u}_{\corresidx{i}}\}$} Following the solver in~\cite{Rusinkiewicz2019}, we fix the source point positions $\newpos{\mathbf{v}}_i^{(k)}$ and update closest points via
\begin{equation}
\rho_i^{(k+1)} = \argmin_{\corresidx{i}\in\{1,...,|\mathcal{U}|\}} \|\mathbf{u}_{\corresidx{i}} - \newpos{\mathbf{v}}_i^{(k)}\|.
\end{equation}
Afterward, we also compute the updated robust weights $\{\pointpairweight{i}^{(k+1)}\}$ according to Eq.~\eqref{eq:robustweight}.

\mypara{Update of $\{\newpos{\mathbf{v}}_i\}$}
After updating $\{\mathbf{u}_{\corresidx{i}}\}$, we fix the closest points $\{\mathbf{u}_{\corresidx{i}^{(k+1)}}\}$, rotations $\{\mathbf{R}_i^{(k)}\}$, and robust weights $\{\pointpairweight{i}^{(k+1)}\}$, and minimize the target function in Eq.~\eqref{eq:FineOptimization} with respect to $\{\newpos{\mathbf{v}}_i\}$.
It is a linear least-squares problem and can be solved via a linear system. Details can be found in Appx.~\red{A-A}. 

\mypara{Update of $\{\mathbf{R}_i\}$}
Finally, we update the rotation matrices $\{\mathbf{R}_i\}$ by fixing $\{\newpos{\mathbf{v}}_i^{(k+1)}\}$ and $\{\mathbf{u}_{\corresidx{i}^{(k+1)}}\}$ and minimizing the target function with respect to $\{\mathbf{R}_i\}$. This is reduced to an independent sub-problem for each $\mathbf{R}_i$:
\begin{align}
    \min_{\mathbf{R}_i} ~& \pointpairweight{i}^{(k+1)} [(\mathbf{R}_i \mathbf{n}_i + \mathbf{n}_{\corresidxiter{i}{k+1}}^t) \cdot (\newpos{\mathbf{v}}_i^{(k+1)} - \mathbf{u}_{\corresidxiter{i}{k+1}})]^2\nonumber\\
    & + \omega \sum_{\mathbf{v}_j\in\mathcal{N}(\mathbf{v}_{i})}  \|(\newpos{\mathbf{v}}_i^{(k+1)} - \newpos{\mathbf{v}}_j^{(k+1)}) - \mathbf{R}_i (\mathbf{v}_i - \mathbf{v}_j) \|^2,\nonumber\\
    \textrm{s.t.}~& \mathbf{R}_i^T \mathbf{R}_i = \mathbf{I},~\det(\mathbf{R}_i) = 1, \label{eq:Rproblem}
\end{align}
where $\omega = \displaystyle\frac{w_{\arap}\cdot|\mathcal{V}|}{|\mathcal{N}(\mathbf{v}_i)|\cdot2|\mathcal{E}|}$.
As far as we are aware, there is no closed-form solution to this problem. For an efficient update, we adopt the idea of majorization-minimization (MM)~\cite{lange2016mm} and minimize a convex surrogate function 
$\surfunc(\mathbf{R}_i \mid \mathbf{R}_i^{(k)})$ that is constructed based on the current variable value $\mathbf{R}_i^{(k)}$. The surrogate function should bound the original target function  $\origfunc$ from above and have the same value as $\origfunc$ at $\mathbf{R}_i^{(k)}$, i.e.,
\begin{equation}
\label{eq:surrogate-condition}
    \begin{aligned}
        & \surfunc(\mathbf{R}_i \mid \mathbf{R}_i^{(k)}) \geq \origfunc(\mathbf{R}_i) \quad \forall~\mathbf{R}_i,\\
        & \surfunc(\mathbf{R}_i^{(k)} \mid \mathbf{R}_i^{(k)}) = \origfunc(\mathbf{R}_i^{(k)}).
    \end{aligned}
\end{equation}
As a result, the minimizer of $\surfunc(\mathbf{R}_i \mid \mathbf{R}_i^{(k)})$ is guaranteed to decrease the original target function $\origfunc$ compared to $\mathbf{R}_i^{(k)}$, unless $\mathbf{R}_i^{(k)}$ is already a local minimum of $\origfunc$ (in which case the minimizer of $\surfunc(\mathbf{R}_i \mid \mathbf{R}_i^{(k)})$ will be exactly $\mathbf{R}_i^{(k)}$)~\cite{lange2016mm}. In the following, we will derive a simple surrogate function that allows for a closed-form solution, enabling fast and effective update of $\{\mathbf{R}_i\}$.

First, we note that the squared term $[(\mathbf{R}_i \mathbf{n}_i + \mathbf{n}_{\corresidxiter{i}{k+1}}^t) \cdot (\newpos{\mathbf{v}}_i^{(k+1)} - \mathbf{u}_{\corresidxiter{i}{k+1}})]^2$ in Eq.~\eqref{eq:Rproblem} can be written as
\begin{equation}
f(\mathbf{R}_i) = [(\mathbf{R}_i\mathbf{n}_i+\mathbf{n}_i^t)\cdot \mathbf{d}]^2,
\label{eq:f-Ri}
\end{equation}
where $\mathbf{d}=\newpos{\mathbf{v}}_i^{(k+1)}-\mathbf{u}_{\corresidx{i}^{(k+1)}}$ and $\mathbf{n}_i^t=\mathbf{n}_{\corresidxiter{i}{k+1}}^t$. Then, we have
\begin{proposition}
\label{prop:f}
If $\mathbf{d}\neq \mathbf{0}$, then $f(\mathbf{R}_i)$ in Eq.~\eqref{eq:f-Ri} satisfies 
\begin{equation}
\label{eq:f-Ri-equal}
f(\mathbf{R}_i) = \|\mathbf{d}\|^2\cdot \min\nolimits_{\mathbf{h}\in\mathcal{P}}\|\mathbf{R}_i\mathbf{n}_i - \mathbf{h}\|^2,
\end{equation}
where $\mathcal{P}$ is a plane containing all vectors $\mathbf{h} \in \mathbb{R}^3$ that satisfy
\begin{equation}
\label{eq:Ri-condition}
(\mathbf{h} + \mathbf{n}_{i}^t)\cdot \mathbf{d} = 0,
\end{equation}
i.e., $f(\mathbf{R}_i)$ is the squared distance from $\mathbf{R}_i\mathbf{n}_i$ to the plane $\mathcal{P}$, scaled by a factor $\|\mathbf{d}\|^2$.
\end{proposition}
\begin{proof}
If $\mathbf{d}\neq \mathbf{0}$, then for any vector $\mathbf{h}$ satisfying~\eqref{eq:Ri-condition} we have 
\begin{equation}
\begin{aligned}
\label{eq:proof-step1}
f(\mathbf{R}_i) &= [(\mathbf{R}_i\mathbf{n}_i-\mathbf{h} + \mathbf{h}+\mathbf{n}_i^t)\cdot \mathbf{d}]^2 \\
&= [(\mathbf{R}_i\mathbf{n}_i-\mathbf{h})\cdot \mathbf{d} + (\mathbf{h}+\mathbf{n}_i^t)\cdot \mathbf{d}]^2 \\
&= [(\mathbf{R}_i\mathbf{n}_i-\mathbf{h})\cdot \mathbf{d}]^2. 
\end{aligned}
\end{equation}
Moreover, the plane $\mathcal{P}$ defined by~\eqref{eq:Ri-condition} has a unit normal vector $\frac{\mathbf{d}}{\|\mathbf{d}\|}$. Since $\mathbf{h}$ belongs to the plane $\mathcal{P}$, we have 
\begin{equation}
\label{eq:proof-step2}
[(\mathbf{R}_i\mathbf{n}_i-\mathbf{h})\cdot \mathbf{d}]^2 = \|\mathbf{d}\|^2\cdot\left[(\mathbf{R}_i\mathbf{n}_i-\mathbf{h})\cdot \frac{\mathbf{d}}{\|\mathbf{d}\|}\right]^2.
\end{equation}
Note that $(\mathbf{R}_i\mathbf{n}_i-\mathbf{h})\cdot \frac{\mathbf{d}}{\|\mathbf{d}\|}$ is the signed distance from $\mathbf{R}_i\mathbf{n}_i$ to the plane $\mathcal{P}$. Therefore,  
\begin{equation}
\label{eq:proof-step3}
\left[(\mathbf{R}_i\mathbf{n}_i-\mathbf{h})\cdot \frac{\mathbf{d}}{\|\mathbf{d}\|}\right]^2 = \min_{\mathbf{h}\in\mathcal{P}}\|\mathbf{R}_i\mathbf{n}_i-\mathbf{h}\|^2. 
\end{equation}
Using Eqs.~\eqref{eq:proof-step1}, \eqref{eq:proof-step2} and \eqref{eq:proof-step3}, we obtain Eq.~\eqref{eq:f-Ri-equal}.
\end{proof}
Based on Eq.~\eqref{eq:f-Ri-equal}, we can construct a surrogate function for $f$ at the  variable value $\mathbf{R}_i^{(k)}$ as:
\begin{equation}
\label{eq:f-Ri-surrogate}
\overline{f}(\mathbf{R}_i|\mathbf{R}_i^{(k)}) = \|\mathbf{d}\|^2\cdot\|\mathbf{R}_i\mathbf{n}_i-\mathbf{h}_*^{(k)}\|^2,
\end{equation}
where $\mathbf{h}_*^{(k)}$ is the closest projection of $\mathbf{R}_i^{(k)}\mathbf{n}_i$ onto the plane $\mathcal{P}$, i.e., $\mathbf{h}_*^{(k)} = \argmin_{\mathbf{h} \in \mathcal{P}}\|\mathbf{R}_i^{(k)}\mathbf{n}_i - \mathbf{h}\|$ (see Appx.~\red{A-B} for a proof that $\overline{f}(\mathbf{R}_i|\mathbf{R}_i^{(k)})$ satisfies the conditions of a surrogate function).

It is easy to show that $\mathbf{h}_*^{(k)}$ can be computed as 
\begin{equation}
\label{eq:Rn-projection}
\mathbf{h}_*^{(k)} = \mathbf{R}_i^{(k)}\mathbf{n}_i - \mathbf{d}\frac{\left(\mathbf{n}_{i}^t+\mathbf{R}_i^{(k)}\mathbf{n}_i\right)\cdot \mathbf{d}}{\|\mathbf{d}\|^2}.
\end{equation}
By replacing $f(\mathbf{R}_i)$ with its surrogate function in Eq.~\eqref{eq:f-Ri-surrogate}, we replace the optimization problem~\eqref{eq:Rproblem} with the following proxy problem:
\begin{equation}
\begin{aligned}
\min_{\mathbf{R}_i}~ & \pointpairweight{i}^{(k+1)}\|\mathbf{d}\|^2\cdot \|\mathbf{R}_i\mathbf{n}_i-\mathbf{h}_*^{(k)}\|^2 \\
 & + \omega \sum_{\mathbf{v}_j\in\mathcal{N}(\mathbf{v}_i)}  \|(\newpos{\mathbf{v}}_i^{(k+1)} - \newpos{\mathbf{v}}_j^{(k+1)}) - \mathbf{R}_i (\mathbf{v}_i - \mathbf{v}_j) \|^2,\nonumber\\
\text{s.t. } & \mathbf{R}_i^T\mathbf{R}_i = \mathbf{I},\quad \det(\mathbf{R}_i) = 1.
\end{aligned}
\end{equation}
This problem has a closed-form solution~\cite{sorkine2017least}:
\begin{equation}
\mathbf{R}_i^{(k+1)} = \rightsvdmat
\begin{bmatrix}
1 &   &    \\
  & 1 &    \\
  &   &     \det(\rightsvdmat\leftsvdmat^T)
\end{bmatrix}
\leftsvdmat^T,
\label{eq:RUpdate}
\end{equation}
where the matrices $\leftsvdmat$, $\rightsvdmat$ are from the SVD 
\[
\mathbf{S} = \leftsvdmat\mathbf{\Sigma}_{\mathbf{S}}\rightsvdmat^T
\]
for the matrix
\begin{equation}
\begin{aligned}
\mathbf{S} = &  
~\pointpairweight{i}^{(k+1)}\|\mathbf{d}\|^2\mathbf{n}_i(\mathbf{h}_*^{(k)})^T \\
&~+ \omega \sum_{\mathbf{v}_j\in\mathcal{N}(\mathbf{v}_i)} (\mathbf{v}_i-\mathbf{v}_j)(\newpos{\mathbf{v}}_i^{(k+1)}-\newpos{\mathbf{v}}_j^{(k+1)})^T\\
= &~ \pointpairweight{i}^{(k+1)} \left(\|\mathbf{d}\|^2 \mathbf{n}_i(\newpos{\mathbf{n}}_i^{(k)})^T
- \mathbf{n}_i \mathbf{d}^T ((\mathbf{n}_{\corresidxiter{i}{k+1}}^t+\newpos{\mathbf{n}}_i^{(k)})\cdot \mathbf{d})\right)\\
&~+ \omega \sum_{\mathbf{v}_j\in\mathcal{N}(\mathbf{v}_i)} (\mathbf{v}_i-\mathbf{v}_j)(\newpos{\mathbf{v}}_i^{(k+1)}-\newpos{\mathbf{v}}_j^{(k+1)})^T
\end{aligned}
\label{eq:SMatrix}
\end{equation}
where $\newpos{\mathbf{n}}_i^{(k)} = \mathbf{R}_i^{(k)}\mathbf{n}_i$. 
Note that although the above derivation is based on Proposition~\ref{prop:f} which requires $\mathbf{d} \neq \mathbf{0}$, the solution in Eq.~\eqref{eq:RUpdate} remains effective when  $\mathbf{d}=\mathbf{0}$: in this case, the target function in~\eqref{eq:Rproblem} reduces to  $$\omega \sum_{\mathbf{v}_j\in\mathcal{N}(\mathbf{v}_{i})}  \|(\newpos{\mathbf{v}}_i^{(k+1)} - \newpos{\mathbf{v}}_j^{(k+1)}) - \mathbf{R}_i (\mathbf{v}_i - \mathbf{v}_j) \|^2,$$ and the matrix $\mathbf{S}$ in Eq.~\eqref{eq:SMatrix} becomes $$\mathbf{S} = \omega \sum_{\mathbf{v}_j\in\mathcal{N}(\mathbf{v}_i)} (\mathbf{v}_i-\mathbf{v}_j)(\newpos{\mathbf{v}}_i^{(k+1)}-\newpos{\mathbf{v}}_j^{(k+1)})^T.$$ Then the matrix $\mathbf{R}_i^{(k+1)}$ in Eq.~\eqref{eq:RUpdate} is exactly the solution to the reduced optimization problem~\cite{sorkine2017least}. Later in Sec.~\ref{sec:ablations}, we will showcase the benefits of this solution for updating $\{\mathbf{R}_i\}$.

\mypara{Termination criteria}
We stop the iteration if at least one of the following conditions is satisfied: (1) the $\ell_2$ norm of the point position changes in an iteration is less than a threshold, i.e., 
$\|\newpos{\mathbf{V}}^{(k+1)} - \newpos{\mathbf{V}}^{(k)}\|/\sqrt{|\mathcal{V}|} < \epsilon$ , where $\epsilon$ is a user-specified parameter (we set 
$\epsilon=10^{-4}$ in our experiments; (2) the number of iterations reaches an upper bound $K$ (we set $K=30$).
Algorithm~\ref{Alg:fine-registration} summarizes our solver for the optimization problem in Eq.~\eqref{eq:FineOptimization}.

\begin{algorithm}[t]
	\caption{Non-rigid registration with robust SP2P distance.}
	\label{Alg:fine-registration}
	\KwIn{$\{\mathbf{v}_i, \mathbf{n}_i\}_{i=1}^{|\mathcal{V}|}$: the source points and normals;\\ ~~~~$\{\mathbf{u}_i, \mathbf{n}_i^t\}_{i=1}^{|\mathcal{U}|}$: the target points and normals;\\ 
    \\~~~~$K$: maximum number of iterations;\\~~~~$\epsilon$: convergence threshold. 
	}
	\KwResult{The deformed point positions $\{\newpos{\mathbf{v}}_i\}_{i=1}^{|\mathcal{V}|}$.
	} 
	\BlankLine
 Set $\mathbf{R}_i^{(0)} = \mathbf{I}$ and $\newpos{\mathbf{v}}_i^{(0)}=\mathbf{v}_i$ for all $i$\;
 $k = 0$\;
	\While{$k < K$ and ${\|\newpos{\mathbf{V}}^{(k+1)}-\newpos{\mathbf{V}}^{(k)}\|}/{\sqrt{|\mathcal{V}|}} <\epsilon$}{
	    For each $i\in\mathcal{V}$, find the closest point $\mathbf{u}_{\corresidxiter{i}{k+1}}$ for $\newpos{\mathbf{v}}_i^{(k)}$\;
	Compute weight $\pointpairweight{i}^{(k+1)}$ with Eq.~\eqref{eq:robustweight}\;
        Compute $\{\newpos{\mathbf{v}}_i^{(k+1)}\}$ via linear system (\red{32})\; 
        Compute $\{\mathbf{R}_i^{(k+1)}\}$ with Eq.~\eqref{eq:RUpdate}\;
        $\newpos{\mathbf{n}}_i^{(k+1)} = \mathbf{R}_i^{(k+1)}\mathbf{n}_i$\;  
        $k = k+1$\;
	}
\end{algorithm}

\subsection{Coarse Alignment Using a Deformation Graph}
\label{sec:coarseAlignment}  

Our numerical solver presented in Sec.~\ref{Sec:numericalSolver} is a local solver that searches for a stationary point near the initial solution. Therefore, proper initialization is crucial for achieving desirable results. To this end, we initialize the solver with a coarse alignment computed using a deformation graph~\cite{sumner2007embedded}. The deformation graph controls the source surface shape with a reduced number of variables, allowing us to efficiently determine a deformation that roughly aligns the two surfaces while preserving the structure of the source shape.

Specifically, to construct a deformation graph, we follow~\cite{yao2022fast} and first uniformly sample a subset $\mathcal{V}_{\mathcal{G}}=\{\mathbf{p}_1,\ldots,\mathbf{p}_{|\mathcal{V}_{\mathcal{G}}|}\}$ from $\mathcal{V}$ as the deformation graph nodes, so that the number of nodes is much smaller than the number of source points. Next, we establish an edge set  $\mathcal{E}_{\mathcal{G}}$ by connecting the neighboring nodes, thereby deriving a deformation graph $\mathcal{G} = \{\mathcal{V}_{\mathcal{G}},\mathcal{E}_{\mathcal{G}}\}$. 
We then assign to each node an affine transformation, represented by a matrix  
$\mathbf{A}_j\in\mathbb{R}^{3\times 3}$ and a vector $\mathbf{t}_j \in\mathbb{R}^3$. 
The transformations at the nodes determine the deformation of each source point $\mathbf{v}_i$ as
\begin{equation}
\label{eq:deformed-pos}
\newpos{\mathbf{v}}_i = \sum_{\mathbf{p}_j\in\mathcal{I}(\mathbf{v}_i)} w_{ij}\cdot (\mathbf{A}_j(\mathbf{v}_i-\mathbf{p}_j)+\mathbf{p}_j + \mathbf{t}_j), 
\end{equation}
where 
$
\mathcal{I}(\mathbf{v}_i) = \{\mathbf{p}_j\in\mathcal{V}_{\mathcal{G}}~|~D(\mathbf{v}_i, \mathbf{p}_j) < R\}
$
is the node set that affects $\mathbf{v}_i$, with $D(\cdot,\cdot)$ denoting the geodesic distance and $R$ being the sampling radius. We set  $R=10\cdot \overline{l}_s$ by default, where $\overline{l}_s$ is the average edge length of the source surface. 
The weight $w_{ij}$ for the influence of $\mathbf{p}_j$ on $\mathbf{v}_i$ is defined as~\cite{yao2022fast}:
\[
w_{ij} = \frac{(1-D^2(\mathbf{v}_i, \mathbf{p}_j)/R^2)^3}{\sum_{\mathbf{p}_k\in\mathcal{I}(\mathbf{v}_i)} (1-D^2(\mathbf{v}_i, \mathbf{p}_k)/R^2)^3}. 
\]
Using the deformation graph, we perform a coarse alignment by optimizing the transformation variables $\{(\mathbf{A}_j, \mathbf{t}_j)\}$ associated with the nodes.

The target function is a combination of the following terms:
\begin{itemize}[leftmargin=*]
    \item \textbf{Alignment and ARAP terms}.
    We use an alignment term $E_{\alignm}^C$ and an ARAP term $E_{\arap}^C$ similar to the ones presented in Sec.~\ref{Sec:alignment} but with two main differences:  (1) the deformed source point position $\newpos{\mathbf{v}}_i$ is computed from the deformation graph according to Eq.~\eqref{eq:deformed-pos}; (2) as only a rough alignment is needed, we apply the alignment term only to a sampled subset $\mathcal{S}$ of the source points $\mathcal{V}$ to reduce computation, i.e.,
    \begin{equation}
     E_{\alignm}^C =  \frac{1}{|\mathcal{S}|}\sum\nolimits_{\mathbf{v}_s\in \mathcal{S}} E_{\alignm}^s, 
    \end{equation}   
    where $E_{\alignm}^s$ is defined in the same way as Eq.~\eqref{eq:alignment}. By default, we set the number of sampling points to 3000.

    \item \textbf{Other regularization terms}.
    To ensure the deformation graph induces a smooth deformation that preserves the structure of the source shape, we introduce two additional regularization terms $E_{\smooth}$ and $E_{\rot}$ from~\cite{yao2022fast} to enforce the smoothness and rigidity of the transformations, respectively. 
\end{itemize}
Using these terms, our optimization problem for coarse alignment can be written as
\begin{equation}
\min_{\{(\mathbf{A}_j, \mathbf{t}_j)\}, \{\mathbf{R}_i\} } E_{\alignm}^C + w_{\arap}^C E_{\arap}^C + w_{\smooth} E_{\smooth}  + w_{\rot} E_{\rot} ,
\label{eq:coarseOptimization}
\end{equation}
where $\{\mathbf{R}_i\}$ are auxiliary rotation matrix variables for the alignment and ARAP terms.
Similar to Sec.~\ref{Sec:numericalSolver}, we solve this problem by alternating updates of the variables using an MM strategy. Details of the solver can be found in Appx.~\red{A-C}.

%% file: experiment.tex
\section{Results}
\label{sec:results}
We conducted comprehensive performance comparisons of the proposed method with the state-of-the-art non-rigid registration methods. In addition, we evaluated the effectiveness of the components in our formulation. This section provides details of our experiment settings and results.

\subsection{Experiment Settings}
To assess the effectiveness and accuracy of our method, we conducted comparisons with several existing methods: the non-rigid ICP method (N-ICP) from~\cite{Amberg2007}; the Welsch function-based formulation from~\cite{yao2022fast} (AMM); the Bayesian Coherent Point Drift method (BCPD)~\cite{hirose2021bayesian} and its variants BCPD++~\cite{hirose2020acceleration} and GBCPD/GBCPD++~\cite{Hirose2022geodesic}. 
Additionally, we compared our method with state-of-the-art learning-based methods, including LNDP~\cite{li2022DeformationPyramid},  SyNoRiM~\cite{huang2023synorim}, and GraphSCNet~\cite{qin2023deep}. 
The comparisons were performed using the open-source implementations of these methods\footnote{\url{https://github.com/Juyong/Fast_RNRR}}$^,$\footnote{\url{https://github.com/yaoyx689/AMM_NRR}}$^,$\footnote{\url{https://github.com/ohirose/bcpd}}$^,$\footnote{\url{https://github.com/rabbityl/DeformationPyramid}}$^,$\footnote{\url{https://github.com/huangjh-pub/synorim}}$^,$\footnote{\url{https://github.com/qinzheng93/GraphSCNet}}.
All methods were evaluated using a synthetic dataset, the DeformingThings4D (DT4D) dataset~\cite{li20214dcomplete}, and real datasets including the articulated mesh animation (AMA) dataset~\cite{vlasic2008articulated}, the SHREC'20 Track dataset~\cite{Dyke2020tracka} {the DFAUST dataset~\cite{bogo2017dfaust} and the BEHAVE dataset~\cite{bhatnagar22behave}}. Additionally, we assessed all optimization-based methods using real datasets including DFAUST~\cite{bogo2017dfaust}, DeepDeform~\cite{bozic2020deepdeform} and the face sequence from~\cite{guo2015robust}, to test their effectiveness for non-rigidly tracking.
For each dataset, we tuned the parameters of each optimization-based method to achieve the best overall performance.  
For learning-based methods, we used the pre-trained models provided by the authors for testing. 
Detailed parameter settings can be found in Appx.~\red{C}.

\input{fig_table_scripts/fig_clean_data}

The optimization-based methods were run on a PC with 32GB of RAM and a 6-core Intel Core i7-8700K CPU at 3.70GHz. 
They were all running on Ubuntu 20.04 LTS system built with Docker and all leveraged multi-thread acceleration on the CPU. 
All learning-based methods were running on a server equipped with an NVIDIA RTX A6000 GPU and utilized CUDA acceleration.
For all problem instances, we scaled the source surface and target surface with the same scaling factor, such that the two surfaces are contained in a bounding box with a unit diagonal length to test all comparison methods.  
To have a clear error scale, we rescale the model to its original size and calculate the error in meters.
In the subsequent presentation of numerical results in tables, we highlight the best results using bold fonts, while underlining the second-best results for clarity and emphasis.
In all result figures, we render the target surface in yellow, while the source surface and the deformed surface are rendered in blue.

For the evaluation metric, we calculated the root-mean-square error between the deformed positions and the ground-truth positions, referring to~\cite{Yao_2020_CVPR,yao2022fast}, to measure the registration error:
\begin{equation}
\text{RMSE} = \sqrt{\frac{1}{|\mathcal{V}|}~{\sum\nolimits_{\mathbf{v}_i\in\mathcal{V}}\|\newpos{\mathbf{v}}_i-\newpos{\mathbf{v}}_i^*\|^2}},
\label{Eq:RMSE}
\end{equation}
where $\newpos{\mathbf{v}}_i^*$ is the ground-truth positions for $\newpos{\mathbf{v}}_i$. 
Since some data do not have ground-truth deformed positions but provide ground-truth correspondences, we also calculate the distance error between the correspondences, similar to~\cite{litany2017deep,cao2023unsupervised}:
\begin{equation}
\correrr = \frac{1}{|\mathcal{C}|}\sum\nolimits_{(i,j)\in\mathcal{C}} \mathcal{D}_{\mathcal{V}}(\mathbf{v}_{\tau_j}, \mathbf{v}_{i}),
\end{equation}
where $\mathcal{C}$ is the set of ground-truth correspondence pairs.
For a given target point $\mathbf{u}_j$, $\mathbf{v}_i$ is the ground-truth corresponding point on the source surface, while  $\mathbf{v}_{\tau_j}$ is the corresponding point obtained from the registration result where its deformed position  $\widehat{\mathbf{v}}_{\tau_j}$ is the closest to $\mathbf{u}_j$ on the deformed source surface. 
When the geodesic distance can be calculated for surface $\mathcal{V}$, $D_{\mathcal{V}}(\cdot,\cdot)$ represents the geodesic distance between two points on surface $\mathcal{V}$. Otherwise, for example, if the source surface is discretized as a mesh with multiple connected parts, $D_{\mathcal{V}}(\cdot,\cdot)$ represents the Euclidean distance between two points. In the following experiments, only the ``crane'' sequences from the AMA dataset~\cite{vlasic2008articulated} with partial overlaps use the Euclidean distance for error calculation.
In addition, following established practices~\cite{litany2017deep,cao2023unsupervised},  we plotted cumulative curves (representing the percentage of correspondences with errors below a varying threshold) and computed the area under these curves (AUC).

\subsection{Comparison with State-of-the-Art Methods}

\input{fig_table_scripts/table_clean}

\input{fig_table_scripts/fig_partial_overlap}
\begin{figure*}[!htb]
    \centering
    \includegraphics[width=\textwidth]{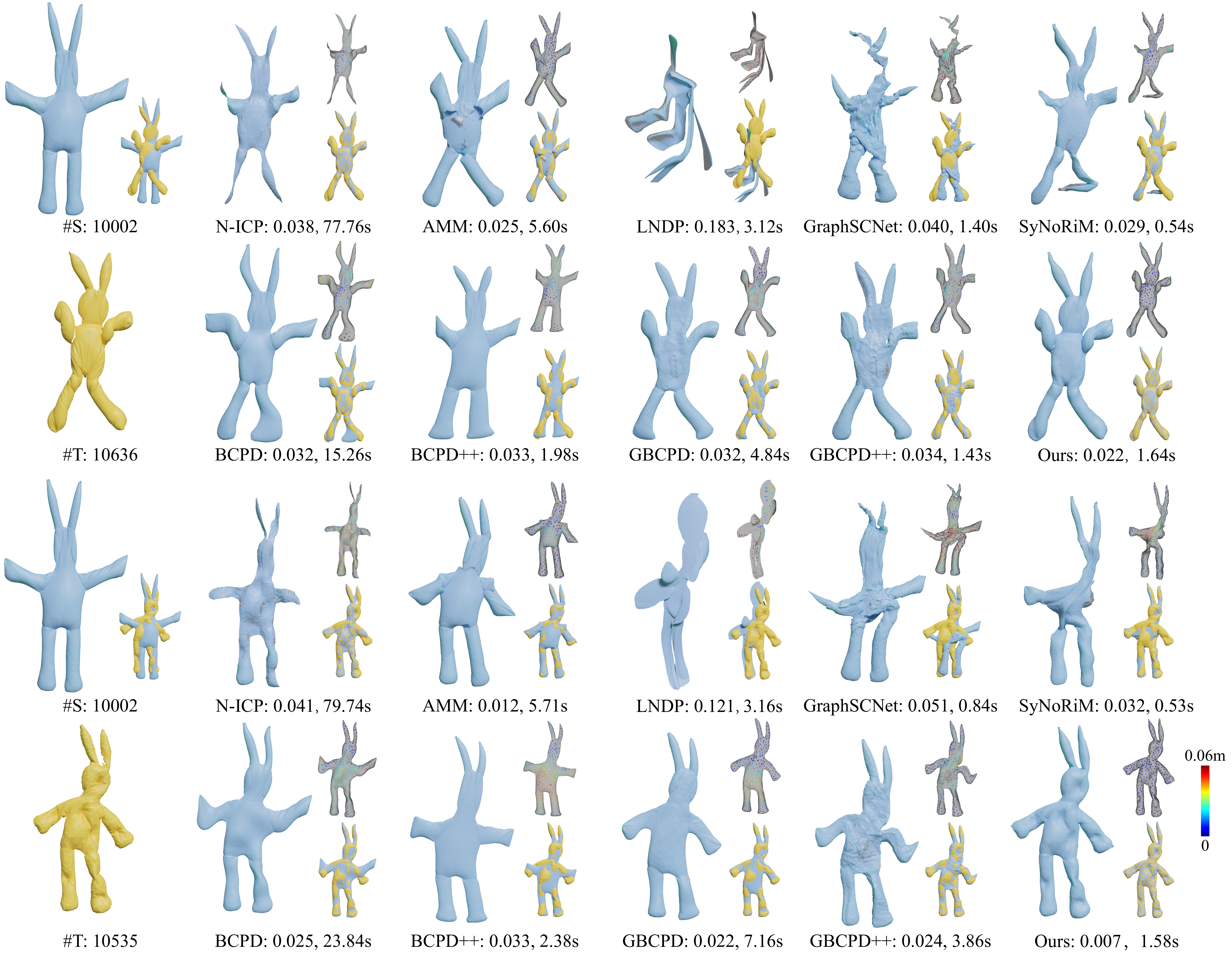}
    \caption{{The results obtained from different methods on two problem instances from the SHREC'20 non-rigid correspondence dataset~\cite{Dyke2020tracka}. For each method, we show the deformed mesh (left), the alignment result (right-bottom), and an error map (right-top) that visualizes the distance between the ground-truth correspondences (Points with no correspondences are marked in gray). We also label the value of $\correrr$ and the computational time for each method.}}
    \label{fig:shrec_compare}
\end{figure*}

\input{fig_table_scripts/table_partial}
\input{fig_table_scripts/fig_behave}

\input{fig_table_scripts/fig_error_curve}

\begin{figure*}[!t]
	\centering
\includegraphics[width=\textwidth]{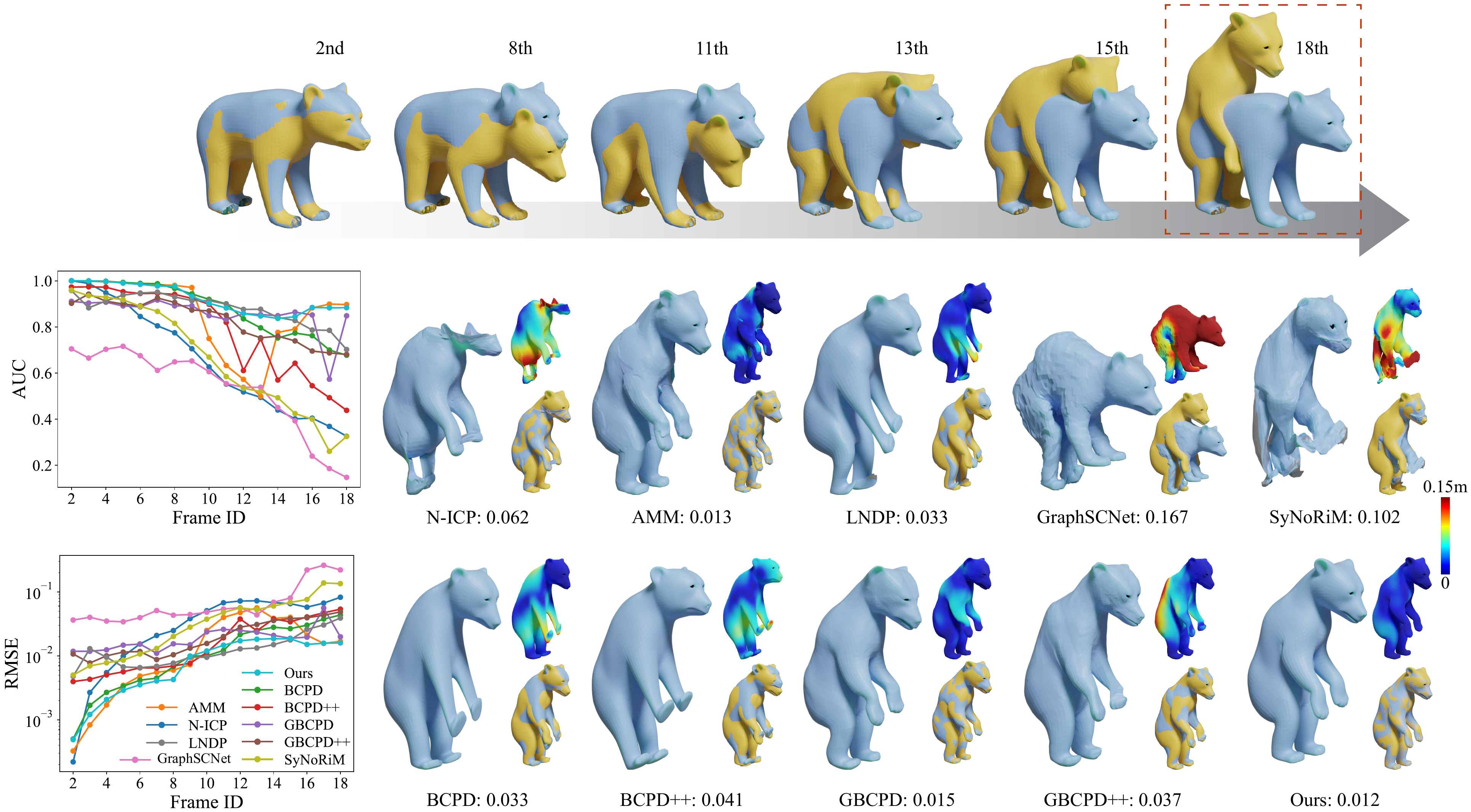}
	\caption{{The results obtained from different methods on DT4D data~\cite{li20214dcomplete} with gradually increasing deformation differences. The top shows the source and target models for several frames. The line graphs in the lower left corner show the results of AUC $\uparrow$ and RMSE $\downarrow$ changing with the indices of the frame respectively. The result of the 18th frame is visualized in the lower right corner.  For each method, we show the deformed mesh (left), alignment result (right-bottom), and an error map (right-top) that visualizes the distance between the ground-truth correspondences, as well as label the RMSE}.}
	\label{fig:dt4d}
\end{figure*}
\input{fig_table_scripts/fig_dfaust_bigmotion}

\begin{figure*}[!t]
	\centering
	\includegraphics[width=\textwidth]{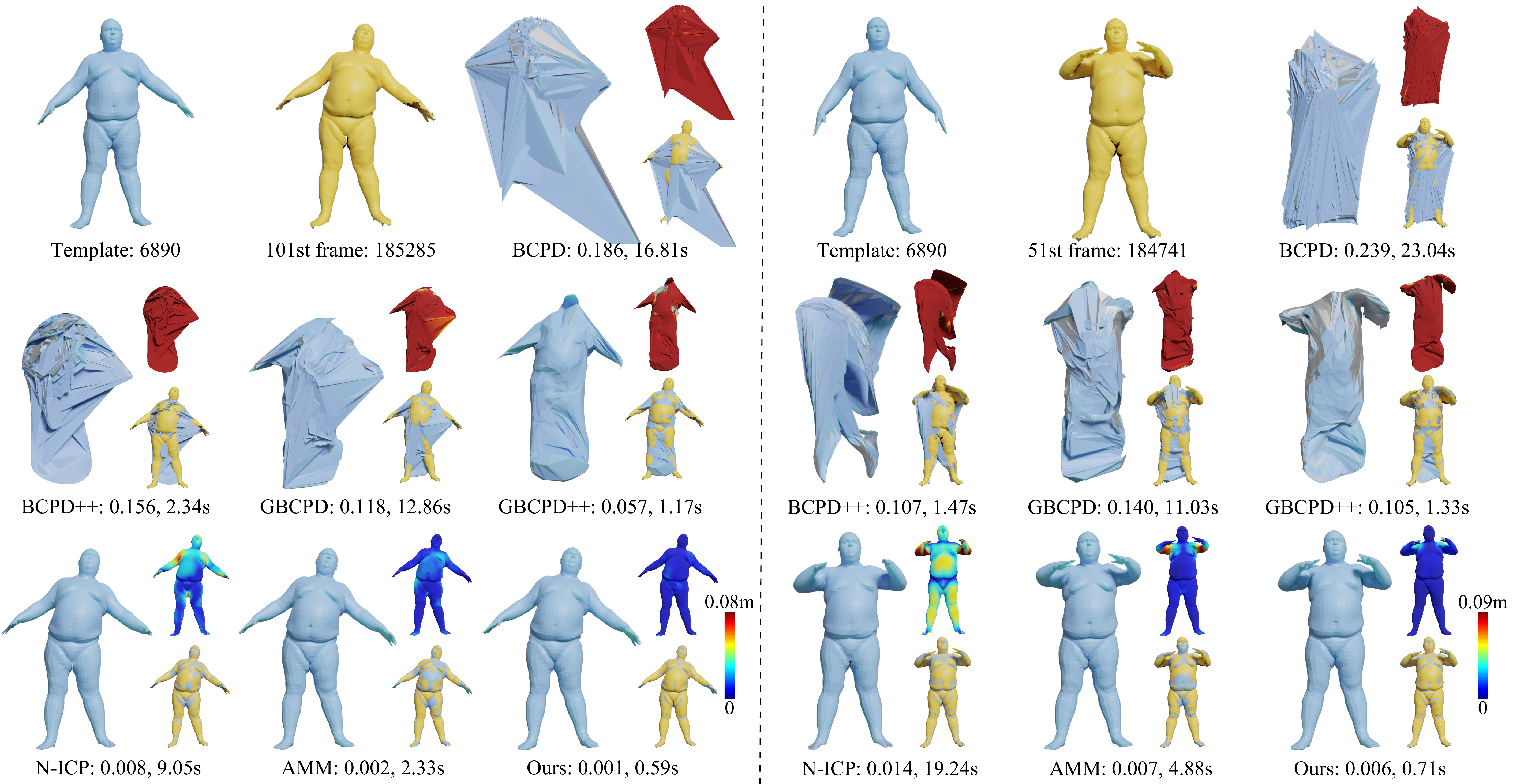}
	\caption{{Results from different methods on two problem instances from ``hips''(left) and ``chicken-wings'' (right) of the DFAUST dataset~\cite{bogo2017dfaust}. For each method, we show the deformed mesh (left), the alignment result (right-bottom), and an error map (right-top) that visualizes the distance between each point and its ground-truth positions, and label the value of $\text{RMSE}$ and the computational time.}}
	\label{fig:dfaust}
\end{figure*}

\mypara{Clean data}
We evaluated different registration methods using the AMA dataset~\cite{vlasic2008articulated} to assess their performance in continuous sequence scenarios. 
The dataset consists of 10 sequences of human continuous motion captured from real-world scenarios, and data have been processed into triangular meshes with the same connectivity structure in each sequence.
Following~\cite{yao2022fast}, we focused on the ``handstand'' and ``march1'' sequences from~\cite{vlasic2008articulated}. For each sequence, we selected 50 pairs of models to test, considering the $i$-th mesh as the source model and the ($i$+2)-th mesh as the target model, where $8\leq i < 48$. 

In Tab.~\ref{Tab:clean_small}, we present the average values of {RMSE, $\correrr$, AUC} and average computation time for each sequence. Additionally, we provide visualizations of two specific cases in Fig.~\ref{fig:human_clean} {and the cumulative curves in Fig.~\ref{fig:auc_curve}}. From the results, it can be observed that our method achieves the {highest accuracy in terms of RMSE, $\correrr$, and AUC} among all compared methods. Furthermore, the average computation time of our method is the second shortest and only longer than SyNoRiM, which is a learning-based method with GPU acceleration.

\mypara{Partial overlaps} 
In practice, many non-rigid registration problems involve surface pairs with partial overlaps that increase the difficulty of registration.
To evaluate the methods on such data, we used the ``crane'' sequence from the AMA dataset~\cite{vlasic2008articulated} as test cases. 
We first selected ten pairs of meshes $\{(m_j^s, m_j^t) \mid j = 1, \ldots, 10\} $ from the crane sequence, each consisting of two adjacent frames in the sequence.
For each mesh pair, we derived nine pairs of point clouds by simulating depth cameras from a fixed view angle $v^{s}$ for the source mesh $m_j^s$, and from nine different view angles for angle $\{v^{t}_1,...,v^{t}_9\}$ for the target mesh $m_j^t$ with increasing deviation from the source view angle.
This results in nine pairs of point clouds for each mesh pair, with decreasing overlap ratios.
we report the average values of the following overlap ratio $o$ for all pairs using the same view angles:
\begin{equation}
\label{eq:overlap_rate}
o = \frac{|\{\mathbf{v}_i \mid \|\newpos{\mathbf{v}}_i^\ast - \mathbf{u}_{\corresidxiter{i}{\ast}}\| < \overline{l}_t / \sqrt{3}\}|}{|\mathcal{V}|},
\end{equation}
where $\newpos{\mathbf{v}}_i^\ast$ is $\mathbf{v}_i$'s position  under the ground-truth deformation,  $\mathbf{u}_{\corresidxiter{i}{\ast}}$ is the closest target point to $\newpos{\mathbf{v}}_i^\ast$, and $\overline{l}_t$ is the average distance between neighboring points on the target shape. That is, $o$ represents the proportion of source points whose distance to the target shape, under the ground-truth deformation, is smaller than a threshold related to the sampling density.
Based on the results shown in Tab.~\ref{Tab:partial_human}, {Fig.~\ref{fig:human-partial} and Fig.~\ref{fig:auc_curve}}, it is evident that the proposed method achieved the highest matching accuracy among all the compared methods.

\input{fig_table_scripts/table_shrec}

In addition, we utilized the models from the SHREC'20 dataset~\cite{Dyke2020tracka} to evaluate the performance of the registration methods. The dataset consists of a complete mesh and 11 partial meshes of a real-world object, where the partial meshes exhibit diverse shapes with fine details, resulting from different deformations (stretch, indent, twist and inflate) and different filling materials. We considered the full mesh as the source surface, and the partial meshes as the target surfaces. 
We compute the $\correrr$ with the provided sparse ground-truth correspondences.
To further evaluate the errors of dense points, we calculate the point-to-point distance $D_{\pp}^{\text{T-S}}$ and the point-to-plane distance $D_{\ppl}^{\text{T-S}}$ from the partial meshes to the deformed source surfaces.
Tab.~\ref{Tab:SHERC20a} shows the numerical results, and Fig.~\ref{fig:shrec_compare} visualizes two specific cases.
We can see that our proposed method achieves notably higher accuracy compared to other methods in the evaluation. This superior performance can be attributed to the high degrees of freedom provided by our dense deformation field, as well as the reweighting scheme that can effectively handle partial overlaps.
Moreover, our method achieves the best speed among optimization-based methods, and only falls behind the GPU-accelerated learning-based methods SyNoRiM and GraphSCNet.

\mypara{Noisy data}
We assessed our method's performance on noisy data using the BEHAVE dataset~\cite{bhatnagar22behave}, which depicts human-object interactions and includes typical outliers produced by depth sensors. From this dataset, we selected the ``Data01'' sequence, containing 35 distinct motion sequences of an individual. Each sequence in BEHAVE provides frame-by-frame human body point clouds and their corresponding fitted SMPL~\cite{SMPL2015} mesh.  In total, we constructed 1586 test cases, using the SMPL mesh of the $i$-th frame as the source model, the point cloud of the $(i+1)$-th frame as the target model, and the SMPL mesh of the $(i+1)$-th frame as the ground-truth mesh for error calculation. A significant challenge arose from the large inter-frame differences within each sequence (as processed by~\cite{bhatnagar22behave}, see Fig.~\ref{fig:behave}), which impacted the performance of all evaluated methods. To mitigate this, we manually annotated 17 landmark corresponding points {(marked on the SMPL mesh and transferred to raw point clouds.)} and integrated a landmark term similar to that proposed in~\cite{Amberg2007} into the registration process. These landmarks were utilized by our method as well as all other compared approaches, with the exception of the BCPD-family of methods, which do not accommodate landmark inputs. As shown in Table~\ref{tab:behave},  Fig.~\ref{fig:behave} and Fig.~\ref{fig:auc_curve}, the comparison results demonstrate that our method outperforms existing techniques.

\input{fig_table_scripts/table_real_data}

\mypara{Large deformation}
To verify the robustness of our method to the magnitude of deformation, we conducted evaluations using data with gradually increasing deformation differences. Specifically, we selected a sequence (``bear3EP\_Agression'') from the DT4D dataset~\cite{li20214dcomplete} {and two sequences (``hips'' and ``jumping-jacks'') from the DFAUST dataset~\cite{bogo2017dfaust}} to test our method. DT4D is a synthetic dataset that includes continuous motion sequences of various animals and humanoids.  
The shapes in each sequence have ground-truth correspondences. We set the first mesh as the source model and the $i$-th mesh as the target model ($i$=2,...,18) for this sequence. 
DFAUST contains multiple continuously moving human scans and the corresponding SMPL mesh~\cite{SMPL2015}. We set the SMPL mesh corresponding to the first frame as the source model and the $i$-th scan as the target model ($i$=2,...,30) to test the performance. The corresponding $i$-th SMPL mesh is used to obtain ground-truth correspondences for calculating the correspondence error $\correrr$.
From  Fig.~\ref{fig:dt4d} and Fig.~\ref{fig:dfaust-bigmotion}, we can observe that as the deformation increases, our method achieves higher or comparable accuracy compared to existing methods.

\mypara{Non-rigid tracking}
To evaluate the practicality of the proposed method, we conducted non-rigid tracking experiments on real-world data, including the DFAUST dataset~\cite{bogo2017dfaust}, the DeepDeform dataset~\cite{bozic2020deepdeform} and the face sequence from~\cite{guo2015robust}. 
For a real scan sequence $\{\mathcal{T}_1, ..., \mathcal{T}_N\}$, we choose a template model $\mathcal{S}_0$, and register it to $\mathcal{T}_1$ and obtain the deformed model $\newpos{\mathcal{S}}_1$. Then we deform $\newpos{\mathcal{S}}_{i}$ to align with $\mathcal{T}_{i+1}$($i=1,...,N-1$) and obtain a mesh sequence $\{\newpos{\mathcal{S}_{i}}\}_{i=1}^{N}$ aligned with $\{\mathcal{T}_{i}\}_{i=1}^{N}$. 
Due to the accumulation of errors in the registration process, this setting increases the difficulty of the problem and makes it easier to demonstrate the performance of different methods. 
Since there is usually a small difference between two adjacent frames, it is generally solved using optimization-based methods. {Since learning-based methods excel in acquiring global correspondences but often struggle with capturing intricate details and exhibit limited stability in continuous registration,} we ignore the comparisons with learning-based methods in these examples.

We utilized the template mesh from the SMPL model matching the first frame as $\mathcal{S}_0$, and deformed it to align the following $i$-th human scan, where $1\leq i < 200$. We also used the template meshes of the SMPL model matched to these scans as the ground-truth meshes. 
We chose four sequences: ``hips'', ``jiggle-on-toes'', ``jumping-jacks'' and ``chicken-wings'' of the identity labeled as 50002 to test, where the deformations of the latter two sequences are much larger than those of the first two sequences. 
We compute {RMSE} between the registered results with the ground-truth meshes, and show their average values and the average running time in Tab.~\ref{Tab:dfaust}. We also provide visualizations of two cases in Fig.~\ref{fig:dfaust}.     

The face sequence from~\cite{guo2015robust} is represented as depth maps and captures facial expressions and muscle movements, which were acquired using a single depth sensor. We utilized the provided template model from~\cite{guo2015robust} as the template model $\mathcal{S}_0$, and deformed it to align with the continuous $299$ point clouds, which were obtained by converting depth maps.

\input{fig_table_scripts/table_dfaust}

\input{fig_table_scripts/table_ablas_all}

\begin{figure}[t]
	\centering
\includegraphics[width=\columnwidth]{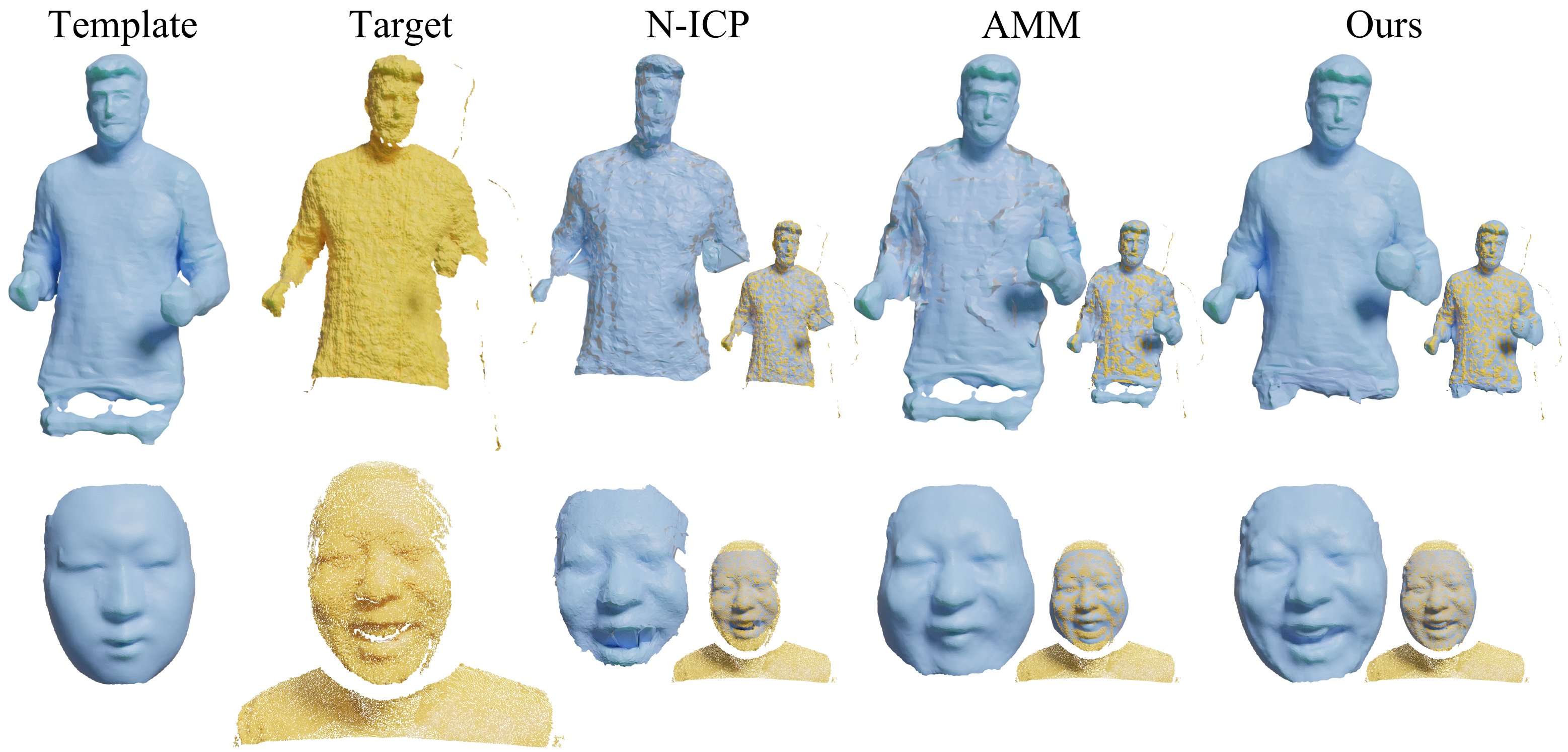}
	\caption{The results obtained from different methods on two problem instances from the DeepDeform dataset~\cite{bozic2020deepdeform}(top: 14th frame) and the face sequence from~\cite{guo2015robust}(bottom: 100th frame). For each method, we show the deformed mesh (left) and alignment result (right).}
	\label{fig:deepdeformface}
\end{figure}

The DeepDeform dataset~\cite{bozic2020deepdeform} is a real RGB-D video dataset, including various scenes such as humans, clothes, animals, etc. Since it does not provide well-defined template meshes, 
{we adopted the reconstructed mesh from~\cite{Cai2022NDR} as our template, first adjusting it to match the initial frame via a rigid transformation, and then deforming it to align with the subsequent 29 point clouds converted from the depth maps.} 
We notice that the BCPD-class methods are too ineffective, so we omit their results.
We show the 2 cases of registration results in Fig.~\ref{fig:deepdeformface} and we can observe that the proposed method is significantly better than other methods. It shows that our method can perform non-rigid registration more stably and reliably.
To more intuitively show the performance, we also render the frame-by-frame rendering results of different methods for these three datasets in the \textit{Supplementary Video}.

\begin{figure}[t]
	\centering
	\includegraphics[width=\columnwidth]{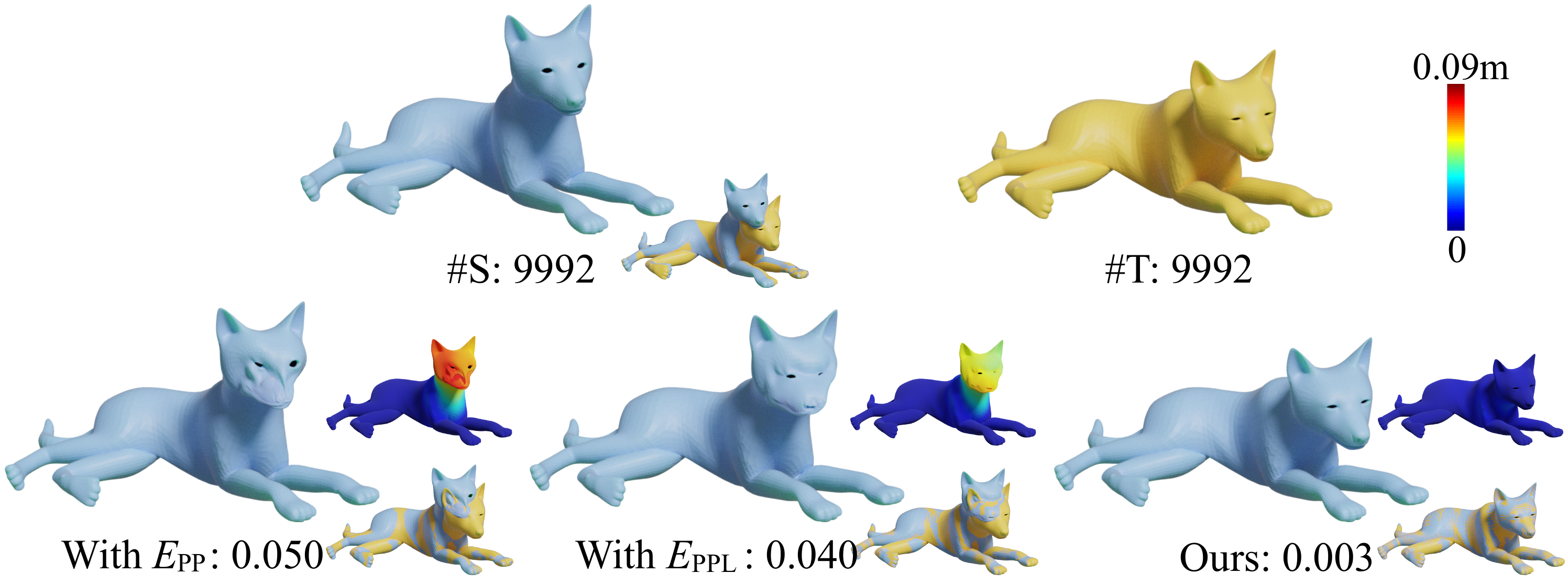}
	\caption{{Comparisons of our method and the variants with point-to-point distance and point-to-plane distance on a problem instance of ``doggieMN5$\_$Sleep'' from~\cite{li20214dcomplete}. For each variant, we show the deformed mesh (left), alignment result (right-bottom), and an error map (right-top) that visualizes the distance between the ground-truth correspondences, as well as label the {RMSE}.}}
	\label{fig:ablas-wosppl}
\end{figure}

\subsection{Effectiveness of Components}
\label{sec:ablations}

To measure the effectiveness of different components in our method, we conducted experiments by either removing a specific term or replacing it with an existing method. We compared these variants on the AMA dataset~\cite{vlasic2008articulated} with small differences (``handstand'' and ``march1'' sequences in Tab.~\ref{Tab:clean_small}) and partial overlaps (``crane'' sequence in Tab.~\ref{Tab:partial_human}). For the ``crane'' sequence, {we present the average values of the accuracy on all examples.} We also performed comparisons on the SHREC'20 non-rigid correspondence dataset~\cite{Dyke2020tracka} and the DT4D dataset~\cite{li20214dcomplete}. We show all numerical comparisons in Tab.~\ref{Tab:ablations} and visualized results in Figs.~\ref{fig:ablas-wosppl}, \ref{fig:ablas-adpweight} and \ref{fig:ablas-Rupdate}.

\mypara{Effectiveness of {symmetrized point-to-plane distance}}
We replaced 
the symmetrized point-to-plane 
distance metric $E_{\symm}$ in our target function by the point-to-point distance metric $E_{\pp}$ (\textit{With} $E_{\pp}$) and the point-to-plane distance metric $E_{\ppl}$ (\textit{With} $E_{\ppl}$) respectively while keeping other components and the reliable weights $\alpha_{i}^{(k)}$ the same. From Tab.~\ref{Tab:ablations} and Fig.~\ref{fig:ablas-wosppl},  
we observe that 
{the symmetrized point-to-plane distance} 
achieves better results than the variants based on $E_{\pp}$ and $E_{\ppl}$. We also show a test case on the DT4D dataset~\cite{li20214dcomplete}, which has more obvious visual differences due to its relatively large deformation.

\begin{figure}[t]
	\centering
	\includegraphics[width=\columnwidth]{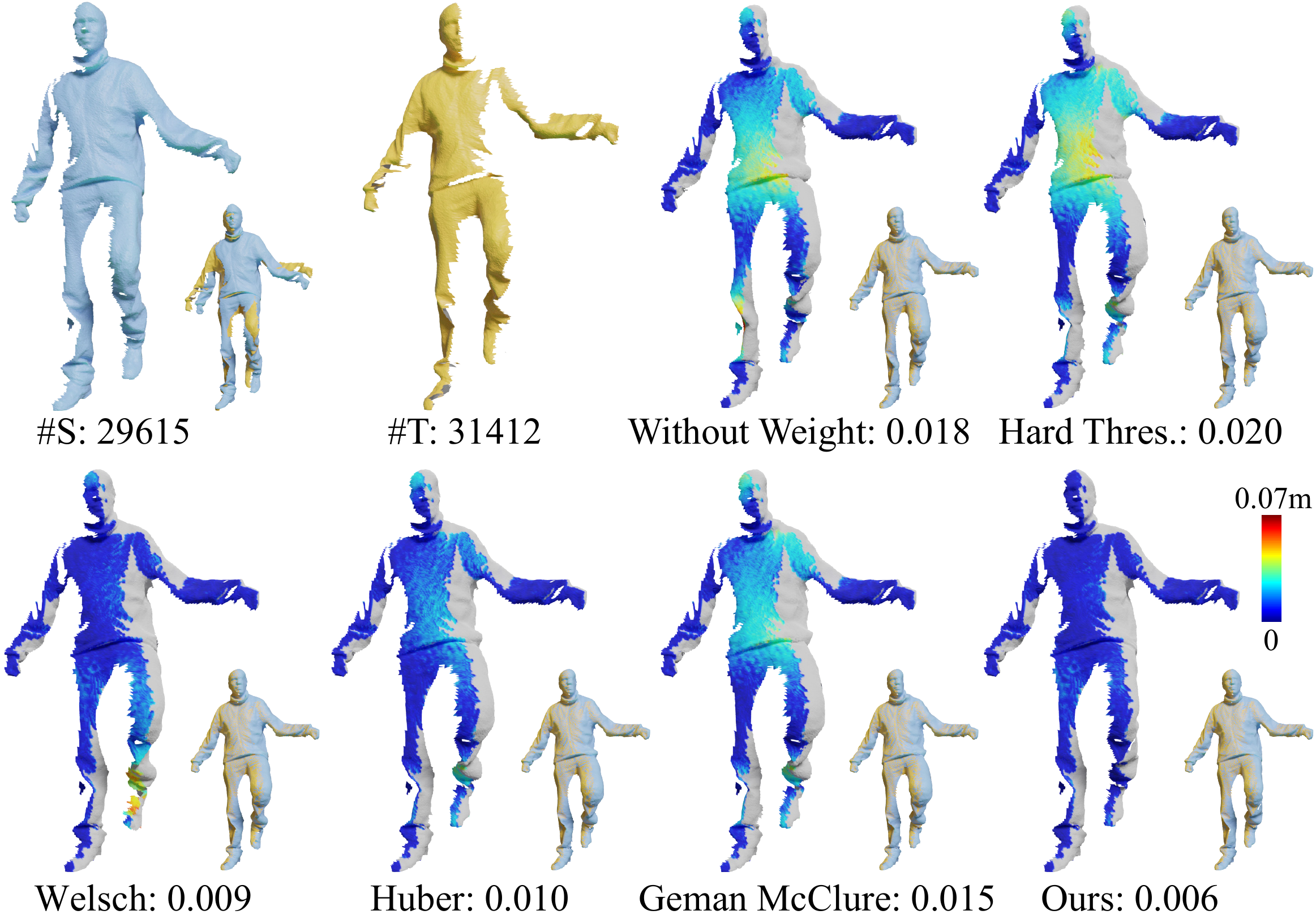}
	\caption{{Comparisons of our method and the variants with different weights for the alignment on a problem instance of partial overlapping data from ``crane'' sequence~\cite{vlasic2008articulated}. For each variant, we show the deformed mesh with an error map that visualizes the distance between the ground-truth correspondences (left) and alignment result (right-bottom), as well as label the RMSE.}}
	\label{fig:ablas-adpweight}
\end{figure}

\mypara{Effectiveness of the robust weights}
Moreover, to test the effectiveness of the adaptive robust weights, we compared with the following strategies: 
\begin{itemize}[leftmargin=*]
\item \textit{Without Weight}: setting $\alpha_i^{(k)}=1$ for all $k$ and $i$ (i.e., no weight);
\item \textit{Hard Thres.}:  setting 
\begin{equation}
\label{eq:reliability-weight-hard-thres}
\alpha_i^{(k+1)}=\left\{
\begin{aligned}
&0, & \text{ if } \|\newpos{\mathbf{v}}_i^{(k)}-\mathbf{u}_{\rho_i^{(k+1)}}\| > 3 \sigma \\ 
& & \text{ or }\newpos{\mathbf{n}}_i^{(k)}\cdot  \mathbf{n}^t_{\corresidx{i}^{(k+1)}}<0, \\
&1, & \text{Otherwise}.
\end{aligned}
\right.
\end{equation}
i.e., a hard thresholding of the weight based on the distance between the corresponding points. 
\end{itemize}
{We also evaluated a common strategy for enhancing robustness, which involves using established robust functions to measure the alignment error. Specifically, we replace $E_{\alignm}^i$ with
\[
\tilde{E}_{\alignm}^i = \beta_i\phi([(\widehat{\mathbf{n}}_i + \mathbf{n}_{\rho_i}^t)\cdot(\widehat{\mathbf{v}}_i - \mathbf{u}_{\rho_i})]), 
\]
where $\beta_i$ is a binary weight designed to filter out point pairs based on normal information; $\beta_i$ is set to 0 if $\widehat{\mathbf{n}}_i\cdot\mathbf{n}_{\rho_i'}^t<0$ and 1 otherwise. $\phi(\cdot)$ is a robust function, and we set it as Welsch loss, Huber loss and Geman-McClure loss respectively to conduct the experiments. These new optimization problems were solved using iteratively reweighted least squares (IRLS) algorithms, drawing on methodologies from~\cite{yao2022fast}\cite{bergstrom2014robust}\cite{adlerstein2025sandro}. }
From Tab.~\ref{Tab:ablations} and Fig.~\ref{fig:ablas-adpweight}, we can observe that our method with robust weights achieves the highest accuracy.

\begin{figure}[t]
	\centering
 \includegraphics[width=\columnwidth]{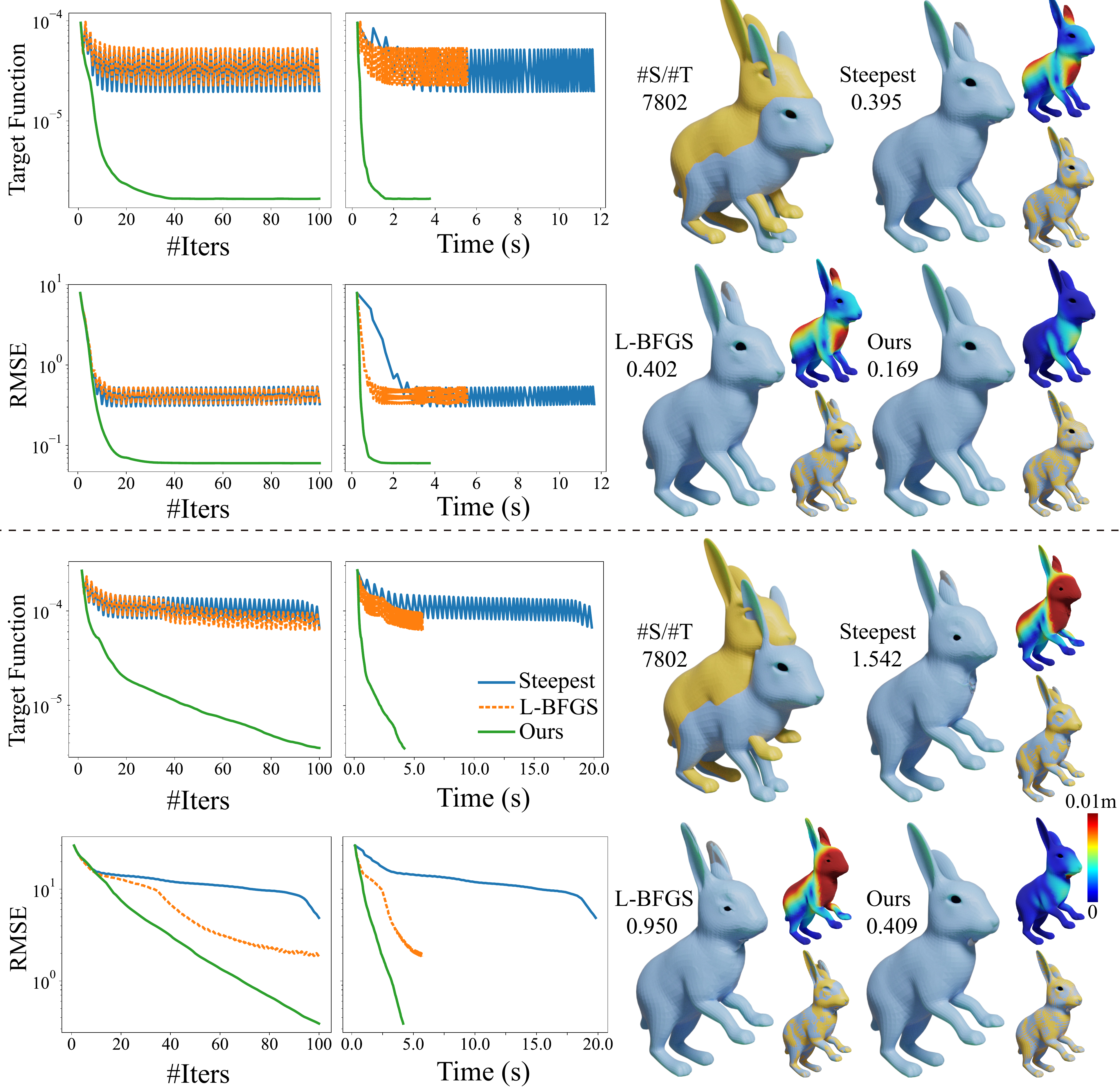}
	\caption{{Comparisons of our method and the variants of updating rotations with the steepest descent method or L-BFGS solver on two problem instances from ``bunnyQ$\_$Attack1'' sequence of~\cite{li20214dcomplete}. 
    For each instance, we show the curves of the objective function $E$ and RMSE changing with the number of iterations or time on the left. On the right, for each variant,  we show the deformed mesh (left), alignment result (right-bottom), and an error map (right-top) that visualizes the distance between the ground-truth correspondences, as well as label the RMSE ($\times 0.01$).}}
	\label{fig:ablas-Rupdate}
\end{figure}

\mypara{Efficiency of solving algorithms} 
To validate the effectiveness of our proposed solution for updating $\{\mathbf{R}_i\}$ in the optimization problem~\eqref{eq:Rproblem}, we compared it with gradient-based solvers that directly optimize the problem in ~\eqref{eq:Rproblem} by the iterative gradient-based optimization method: the steepest descent method and the L-BFGS algorithm. Since $\{\mathbf{R}_i\}$ needs to be in the rotation space, we transformed them in Lie algebra space $so(3)$ and solved them accordingly. Specifically, we replaced $\mathbf{R}_i\in\mathcal{R}$ with $\mathbf{r}_i=[r_1,r_2,r_3]^T\in so(3)$ via Rodrigues' rotation formula. 
By utilizing the transformed variables ${\mathbf{r}_i}$ as independent variables and removing the equality constraint in problem~\eqref{eq:Rproblem}, we can convert it into an unconstrained optimization problem. This allows us to employ optimization algorithms such as the steepest descent method or the L-BFGS algorithm to solve it.
To ensure optimal convergence of the algorithm, we set a maximum limit of $10^{3}$ iterations for solving each sub-problem of $\mathbf{r}_i$. Additionally, the algorithm is terminated prematurely if the difference in objective function values between two consecutive iterations falls below a threshold of $10^{-10}$. 
To eliminate potential discrepancies caused by adaptive weights, we set $\alpha_i^{(k)}=1$ for any $k$ and $i$ in this part.
Fig.~\ref{fig:ablas-Rupdate} displays the curves depicting the changes in the objective function values $E=E_{\alignm} + w_{\arap}E_{\arap}$ in Eq.~\eqref{eq:FineOptimization} and {RMSE} over the iterations or time. 
Due to the high degree of nonlinearity of the problem, gradient-based optimization methods are susceptible to getting trapped in local optima. 
It is important to note that the curve may exhibit fluctuations since using the closest point as the corresponding point does not necessarily result in a reduction of the symmetrized point-to-plane distance.
From Fig.~\ref{fig:ablas-Rupdate}, it is evident that our proposed method achieves stable and rapid convergence toward the optimal solution.  

%% file: fig_table_scripts/fig_clean_data.tex
\begin{figure*}[t]
	\centering
	\includegraphics[width=\textwidth]{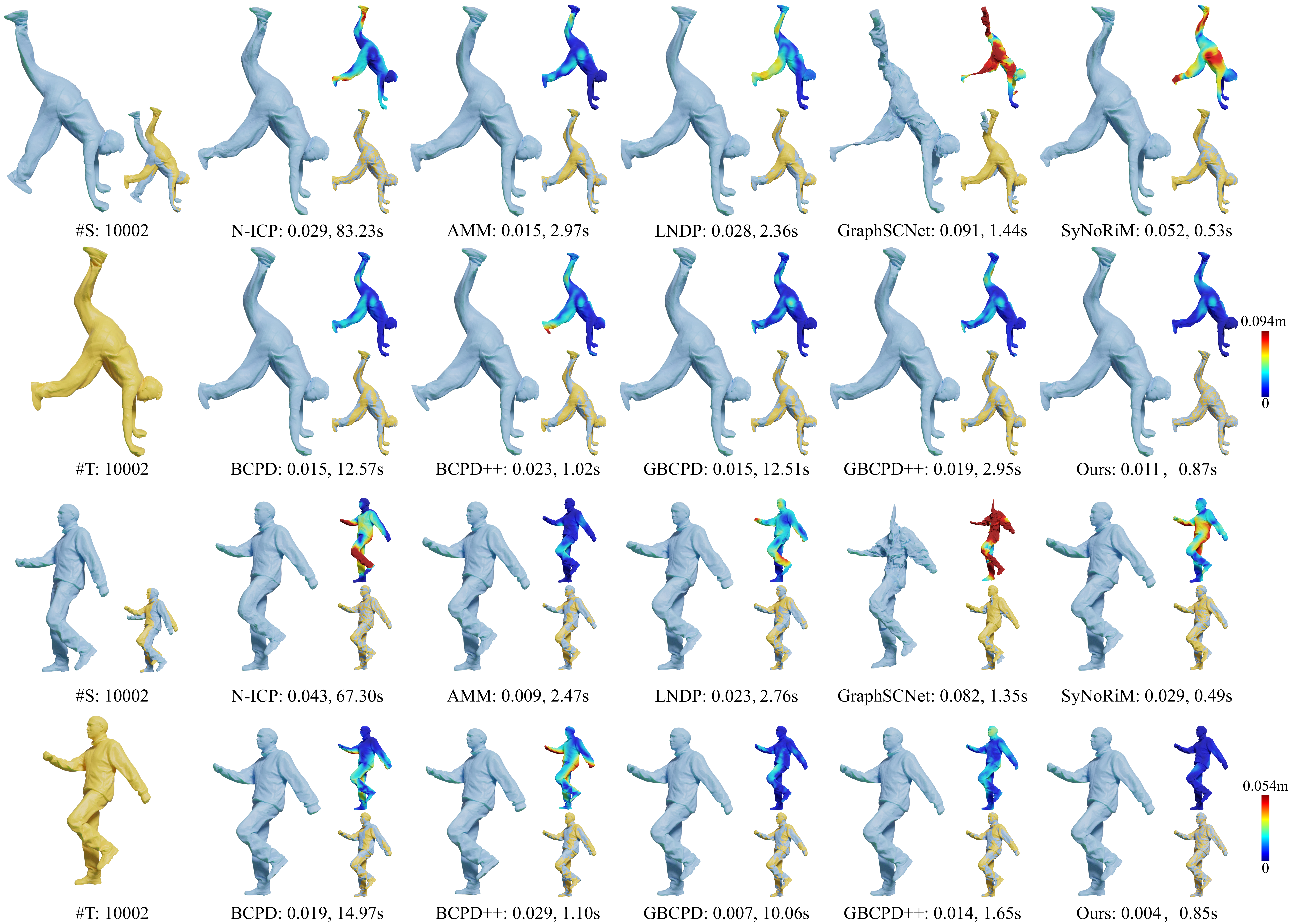}
    \vspace{-1.5em}
	\caption{{Results from different methods on two problem instances from the ``handstand''(top) and ``march1''(bottom) sequences from~\cite{vlasic2008articulated}. For each method, we show the deformed mesh (left), alignment result (right-bottom), and an error map (right-top) that visualizes the distance between each point and its ground-truth corresponding points, as well as the RMSE and the computational time.}}
	\label{fig:human_clean}
\end{figure*}

%% file: fig_table_scripts/table_clean.tex
\begin{table}[t!]
	\caption{
		{Average values of RMSE ($\times 0.01$) $\downarrow$  / $\correrr$ ($\times 0.01$) $\downarrow$ / AUC $\uparrow$ and computational time (s) on ``handstand" and ``march1" sequences from~\cite{vlasic2008articulated}.}}  
	\label{Tab:clean_small}
	\setlength{\tabcolsep}{1.8pt}
    \renewcommand{\arraystretch}{0.8}
	\centering
	\begin{small}
		\begin{tabular}{ c | c  c| c c}
			\toprule 
            \multirow{2}{*}{\makecell[c]{Method}} & \multicolumn{2}{c|}{handstand}& \multicolumn{2}{c}{march1} \\\cmidrule(r){2-3} \cmidrule(r){4-5}
			& Accuracy &  Time & Accuracy & Time \\
			\midrule
N-ICP~\cite{Amberg2007} & 3.03 / 2.17 / 0.81  & 66.58 & 1.67 / 1.18 / 0.89 & 39.84 \\
AMM~\cite{yao2022fast} & 1.19 / \underline{0.62} / \underline{0.94}  & 2.35 & 0.70 / 0.21 / 0.98 & 1.51 \\
BCPD~\cite{hirose2021bayesian} & 1.29 / 0.79 / 0.92  & 4.97 & 1.18 / 0.76 / 0.93 & 5.40 \\
BCPD++~\cite{hirose2020acceleration} & 1.71 / 1.31 / 0.87  & 3.21 & 1.74 / 1.29 / 0.88 & 3.33 \\
GBCPD~\cite{Hirose2022geodesic} & \underline{1.14} / \underline{0.62} / \underline{0.94}  & 10.07 & \underline{0.55} / \underline{0.15} / \underline{0.99} & 9.53 \\
GBCPD++~\cite{Hirose2022geodesic} & 1.46 / 1.09 / 0.89  & 2.68 & 1.08 / 0.73 / 0.93 & 2.38 \\
GraphSCNet~\cite{qin2023deep} & 9.44 / 7.61 / 0.40  & 1.95 & 7.23 / 6.96 / 0.46 & 1.85 \\
LNDP~\cite{li2022DeformationPyramid} & 2.23 / 1.93 / 0.81  & 1.90 & 1.85 / 1.39 / 0.87 & 1.43 \\
SyNoRiM~\cite{huang2023synorim} & 3.20 / 2.83 / 0.73  & \textbf{0.54} & 2.20 / 1.93 / 0.83 & \textbf{0.48} \\ \midrule 
Ours & \textbf{0.86} / \textbf{0.38} / \textbf{0.96} & \underline{0.96} & \textbf{0.26} / \textbf{0.03} / \textbf{1.00} & \underline{0.80} \\
			\bottomrule
		\end{tabular}
	\end{small}
\end{table}

%% file: fig_table_scripts/fig_partial_overlap.tex
\begin{figure*}[!t]
	\centering
   \includegraphics[width=1\textwidth]{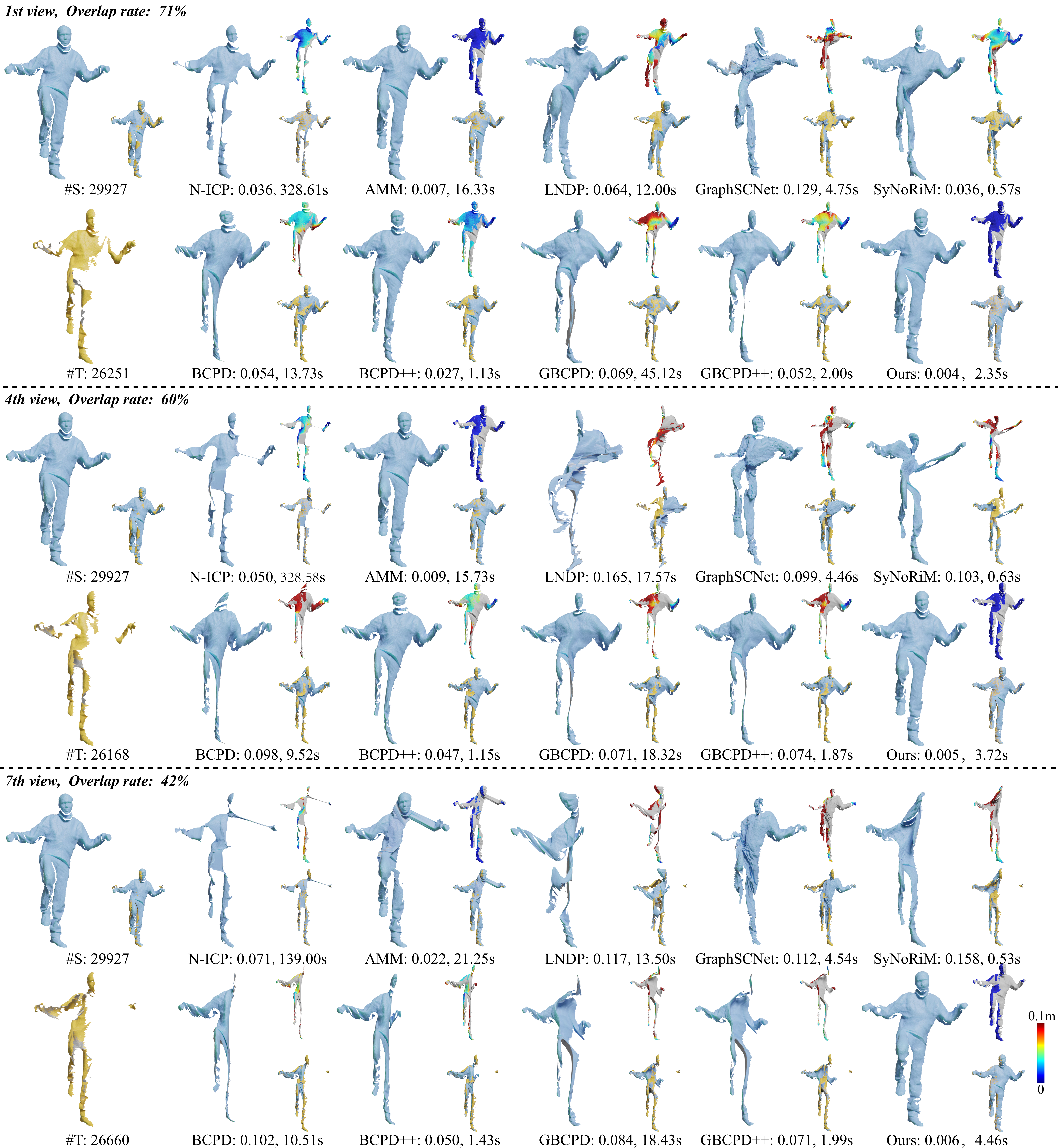}
	\caption{{The results obtained from different methods on problem instances with partial overlaps from~\cite{vlasic2008articulated}. For each method, we show the deformed mesh (left), the alignment result (right-bottom), and an error map (right-top) that visualizes the distance between the ground-truth correspondences in the overlapping area defined by Eq.~\eqref{eq:overlap_rate}  (Points with no correspondence are marked in gray), and label the value of $\correrr$ and the computational time.}}
	\label{fig:human-partial}
\end{figure*}

%% file: fig_table_scripts/table_partial.tex
\begin{table*}[!htb]
	\caption{
		{Mean values of $\correrr\downarrow$ ($\times 0.001$) and AUC$\uparrow$ on ``crane'' sequence from AMA dataset~\cite{vlasic2008articulated} with partial overlaps. Each column (View Angle Pair(the average overlap radio $o$ (\%))) shows the statistics for a particular pair of view angles.}}
	\label{Tab:partial_human}
	\setlength{\tabcolsep}{1.5pt}
    \renewcommand{\arraystretch}{0.8}
	\centering
	\begin{small}
		\begin{tabular}{ c | c  c  c  c  c c c c c}
			\toprule
          Method & 1 (76.30) & 2 (73.55) & 3 (70.94) & 4 (67.98) & 5 (64.43) & 6 (60.71) & 7 (56.60) & 8 (52.50) & 9 (48.61) \\
			\midrule
N-ICP~\cite{Amberg2007} & 3.04 / 0.71 &3.23 / 0.69 &3.28 / 0.68 &3.42 / 0.67 &3.64 / 0.65 &3.98 / 0.62 &4.26 / 0.60 &4.62 / 0.57 &4.79 / 0.56 \\
AMM~\cite{yao2022fast} & \underline{0.79} / \underline{0.92}  & \underline{0.79} / \underline{0.92}  & \underline{0.83} / \underline{0.92}  & \underline{0.85} / \underline{0.92}  & \underline{1.05} / \underline{0.90}  & \underline{0.99} / \underline{0.90}  & \underline{1.17} / \underline{0.89}  & \underline{1.43} / \underline{0.87}  & \underline{2.00} / \underline{0.83}  \\
BCPD~\cite{hirose2021bayesian} & 6.49 / 0.51  & 7.12 / 0.45  & 7.80 / 0.41  & 9.78 / 0.29  & 10.68 / 0.24  & 11.37 / 0.22  & 12.34 / 0.20  & 12.41 / 0.18 
 & 13.36 / 0.17  \\
BCPD++~\cite{hirose2020acceleration} & 3.30 / 0.70  & 3.54 / 0.68  & 3.70 / 0.66  & 4.29 / 0.61  & 4.50 / 0.59  & 4.85 / 0.58  & 6.02 / 0.47  & 6.48 / 0.43  & 7.81 / 0.38  \\
GBCPD~\cite{Hirose2022geodesic} & 7.59 / 0.43  & 9.02 / 0.34  & 9.68 / 0.32  & 9.67 / 0.30  & 10.33 / 0.27  & 10.57 / 0.25  & 10.88 / 0.24  & 11.39 / 0.23  & 11.96 / 0.22  \\
GBCPD++~\cite{Hirose2022geodesic} & 7.30 / 0.41  & 8.09 / 0.37  & 8.34 / 0.34  & 8.52 / 0.34  & 9.61 / 0.29  & 9.86 / 0.27  & 10.10 / 0.26  & 10.60 / 0.25  & 11.21 / 0.23  \\
GraphSCNet~\cite{qin2023deep} & 11.55 / 0.21  & 11.78 / 0.19  & 12.34 / 0.17  & 12.91 / 0.17  & 13.23 / 0.18  & 12.78 / 0.17  & 13.60 / 0.16  & 13.21 / 0.15  & 12.88 / 0.15  \\
LNDP~\cite{li2022DeformationPyramid} & 9.51 / 0.44  & 11.67 / 0.39  & 10.85 / 0.38  & 14.28 / 0.32  & 14.33 / 0.35  & 13.76 / 0.34  & 17.33 / 0.28  & 17.62 / 0.30  & 13.73 / 0.35  \\
SyNoRiM~\cite{huang2023synorim} & 3.89 / 0.65  & 4.42 / 0.60  & 5.01 / 0.57  & 4.79 / 0.56  & 6.41 / 0.49  & 7.35 / 0.45  & 6.85 / 0.46  & 8.93 / 0.40  & 
11.17 / 0.35  \\\midrule
Ours & \textbf{0.41} / \textbf{0.96}  & \textbf{0.42} / \textbf{0.96}  & \textbf{0.42} / \textbf{0.96}  & \textbf{0.44} / \textbf{0.96}  & \textbf{0.46} / \textbf{0.96}  & \textbf{0.48} / \textbf{0.95}  & \textbf{0.49} / \textbf{0.95}  & \textbf{0.50} / \textbf{0.95}  & \textbf{0.57} / \textbf{0.95}  \\
			\bottomrule
		\end{tabular}
	\end{small}
\end{table*}

%% file: fig_table_scripts/fig_behave.tex
\begin{figure*}[t]
	\centering
	\includegraphics[width=0.95\textwidth]{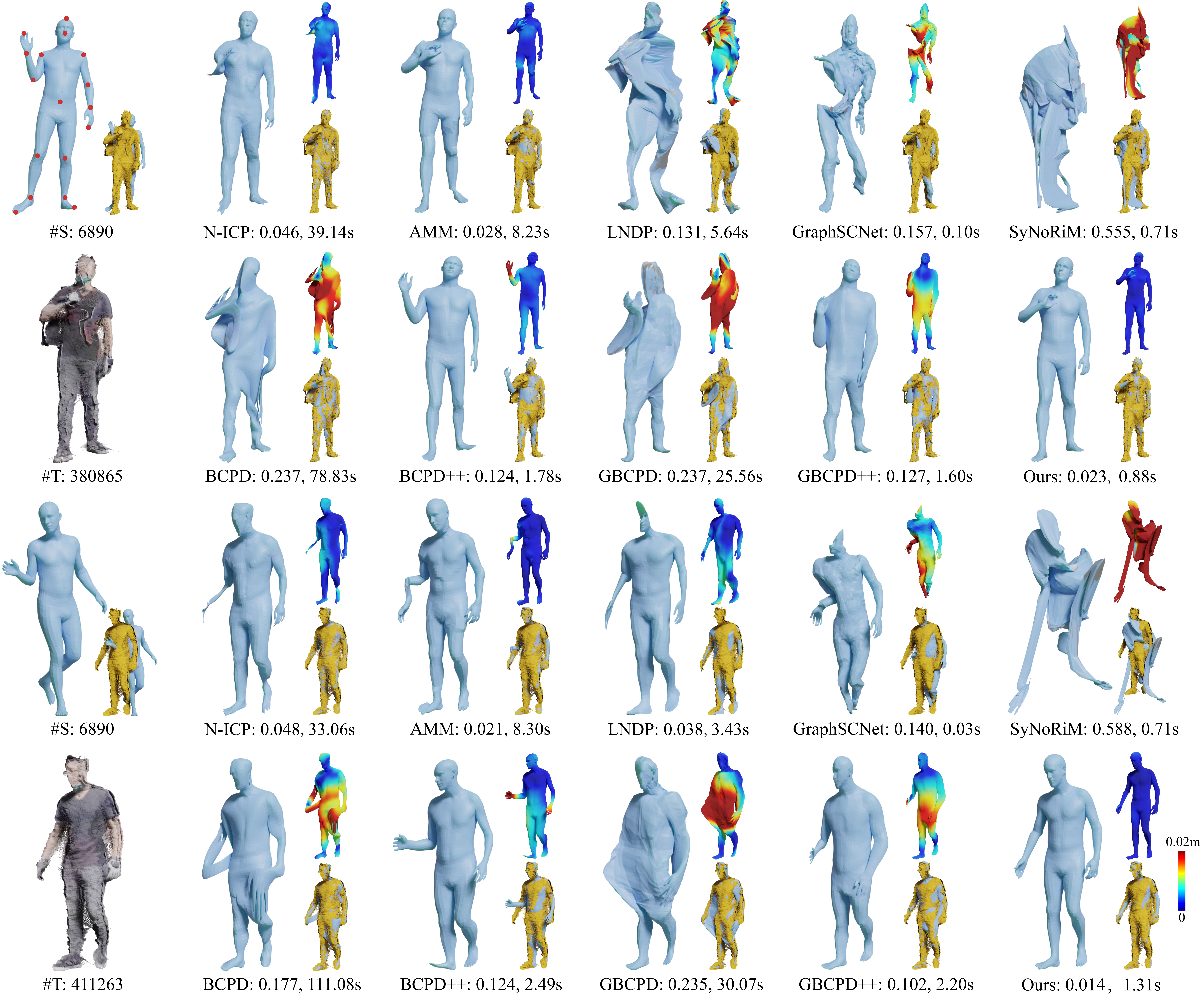}
	\caption{{The results obtained from different methods on two problem instances {from} the BEHAVE dataset~\cite{bhatnagar22behave}. For each method, we show the deformed mesh (left), the alignment results (right-bottom), and an error map (right-top) that visualizes the distance between each point and the ground-truth positions (right-top). We also label the value of RMSE and the computational time. The red dots on the source surface in the top-left corner {mark} the locations of the landmark points.}}
	\label{fig:behave}
\end{figure*}

%% file: fig_table_scripts/fig_error_curve.tex
\begin{figure}[t!]
	\centering
	\includegraphics[width=\columnwidth]{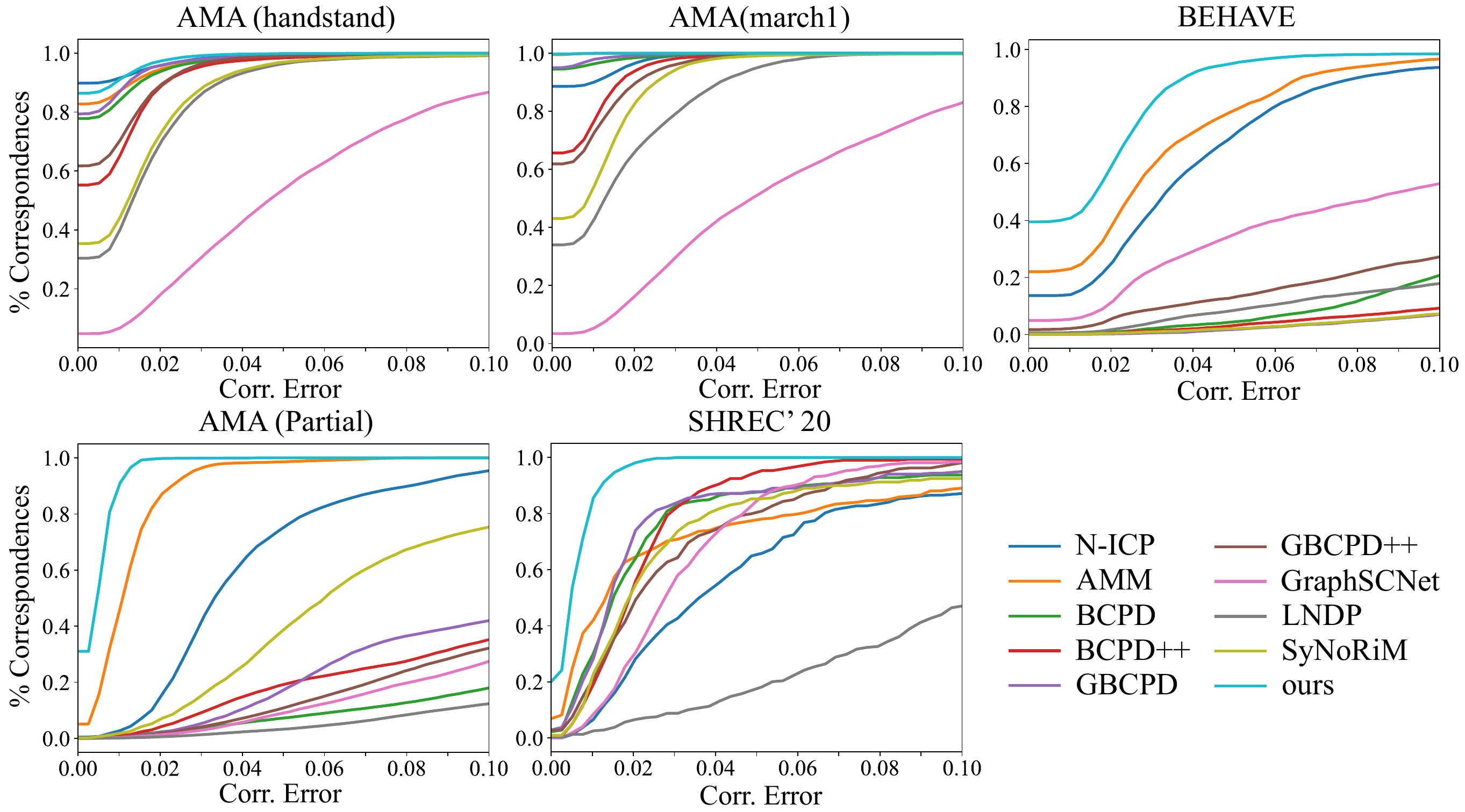}
	\caption{{The cumulative curves for the correspondence error $\correrr$ of different methods on various datasets.}}
	\label{fig:auc_curve}
\end{figure}

%% file: fig_table_scripts/fig_dfaust_bigmotion.tex
\begin{figure*}[t]
	\centering
	\includegraphics[width=\textwidth]{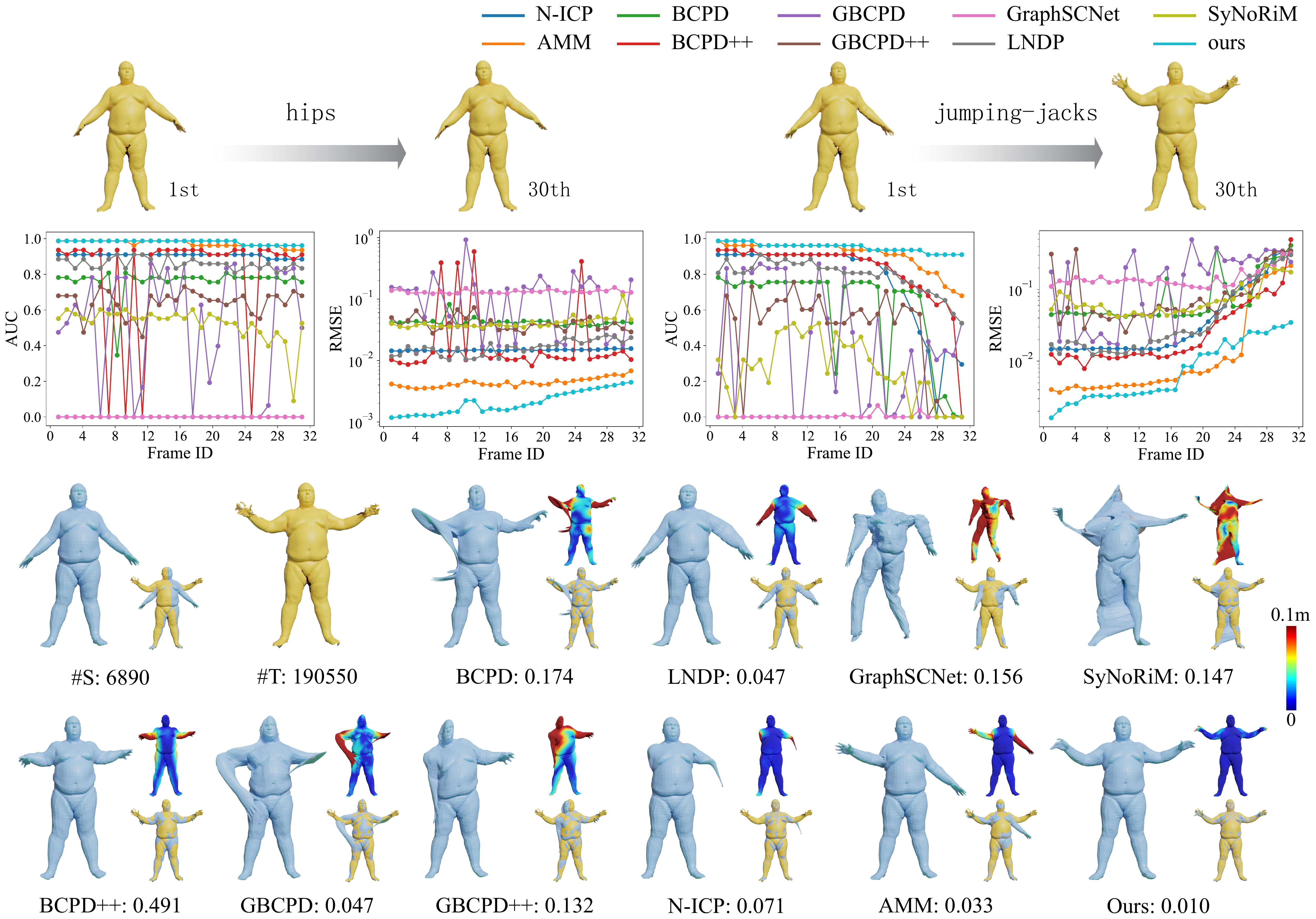}
	\caption{{The results obtained from different methods on DFAUST dataset~\cite{bogo2017dfaust} with gradually increasing deformation difference. The line graphs in the top show the results of AUC $\uparrow$ and RMSE $\downarrow$ changing with the indices of the frame respectively in ``hips'' and ``jumping-jacks'' sequences. The result of the 29th frame in ``jumping-jacks'' is visualized in the bottom. For each method, we show the deformed mesh (left), alignment result (right-bottom), and an error map (right-top) that visualizes the distance between the ground-truth correspondences, as well as label the RMSE.}}
	\label{fig:dfaust-bigmotion}
\end{figure*}

%% file: fig_table_scripts/table_shrec.tex
\begin{table}[t]
	\caption{{Mean values of $D^{\text{T-S}}_{\pp} (\times 10^{-5})$, $D^{\text{T-S}}_{\ppl}(\times 10^{-5})$, $\correrr$ ($\times 0.01$), AUC and average computational time (s) using different methods on the SHREC'20 non-rigid correspondence dataset~\cite{Dyke2020tracka}.}}
	\label{Tab:SHERC20a}
		\setlength{\tabcolsep}{3pt}
        \renewcommand{\arraystretch}{0.8}
	\centering
	\begin{small}
		\begin{tabular}{ c | c  c c c  c}
                \toprule
			\multicolumn{1}{c|}{Method} & $D^{\text{T-S}}_{\pp} \downarrow$ & $D^{\text{T-S}}_{\ppl} \downarrow$ & $\correrr\downarrow$ & AUC $\uparrow$ & Time \\
            \midrule
N-ICP~\cite{Amberg2007} & 66.43 & 37.12 & 4.59 & 0.58 &  79.59 \\
AMM~\cite{yao2022fast} & 16.81 & 13.05 & 2.70 & 0.76 &  5.95 \\
BCPD~\cite{hirose2021bayesian} & 6.81 & 4.56 & 3.23 & 0.69 &  16.80 \\
BCPD++~\cite{hirose2020acceleration} & 10.11 & 6.70 & 3.26 & 0.68 &  2.05 \\
GBCPD~\cite{Hirose2022geodesic} & \underline{1.07} & 0.58 & \underline{2.79} & \underline{0.73} &  7.17 \\
GBCPD++~\cite{Hirose2022geodesic} & 1.12 & \underline{0.56} & 2.98 & 0.72 &  1.84 \\
GraphSCNet~\cite{qin2023deep} & 37.10 & 23.99 & 5.33 & 0.52 &  \underline{1.59} \\
LNDP~\cite{li2022DeformationPyramid} & 203.71 & 114.46 & 14.11 & 0.23 &  3.06 \\
SyNoRiM~\cite{huang2023synorim} & 47.03 & 28.27 & 8.58 & 0.50 &  \textbf{0.54} \\ \midrule
Ours & \textbf{0.41} & \textbf{0.10} & \textbf{1.47} & \textbf{0.86} &  1.69 \\
            \bottomrule
		\end{tabular}
	\end{small}
\end{table}

%% file: fig_table_scripts/table_real_data.tex
\begin{table}[t]
	\caption{{Mean values of $\text{RMSE}$($\times 0.01$), $\correrr$($\times 0.01$), AUC and average computational time (s) using different methods on the BEHAVE dataset~\cite{bhatnagar22behave}.}}
	\label{tab:behave}
		\setlength{\tabcolsep}{3pt}
        \renewcommand{\arraystretch}{0.8}
	\centering
	\begin{small}
		\begin{tabular}{ c | c  c c c}
                \toprule
			\multicolumn{1}{c|}{Method} & $\text{RMSE} \downarrow$ & $\correrr \downarrow$ & AUC $\uparrow$ & Time \\
            \midrule
N-ICP~\cite{Amberg2007} &\underline{4.04} & \underline{4.51} & 0.66 & 37.15 \\
AMM~\cite{yao2022fast}  &4.11 & 4.73 & \underline{0.68} & 8.32 \\
BCPD~\cite{hirose2021bayesian}  &23.53 & 33.89 & 0.17 & 82.34\\
BCPD++~\cite{hirose2020acceleration} & 15.94 & 21.22 & 0.40 & 2.52\\
GBCPD~\cite{Hirose2022geodesic}  & 26.99 & 36.93 & 0.16 & 47.76\\
GBCPD++~\cite{Hirose2022geodesic}  &18.12 & 25.12 & 0.31 & 2.14\\
GraphSCNet~\cite{qin2023deep}  & 16.19 & 20.76 & 0.22 & \textbf{0.04}\\
LNDP~\cite{li2022DeformationPyramid}  & 10.81 & 15.69 & 0.42 & 2.17 \\
SyNoRiM~\cite{huang2023synorim} & 74.32 & 56.86 & 0.03 & \underline{0.71} \\ \midrule
Ours  & \textbf{2.37} & \textbf{2.62} & \textbf{0.81} & 1.13\\ 
            \bottomrule
		\end{tabular}
	\end{small}
\end{table}

%% file: fig_table_scripts/table_dfaust.tex
\begin{table}[t]
	\caption{{Mean values of RMSE$(\times 0.01)$ and average computational time (s) using different methods on the DFAUST dataset~\cite{bogo2017dfaust}. Due to error accumulation, N-ICP only obtained valid results on some sequences (195 for ``jiggle-on-toes'', 74 for ``jumping-jacks'', and 97 for ``chicken-wings''). We only make statistics on these valid values.}}
	\label{Tab:dfaust}
		\setlength{\tabcolsep}{1pt}
	\centering
    \renewcommand{\arraystretch}{0.8}
	\begin{small}
		\begin{tabular}{ c| c c c c}
            \toprule 
            Method & hips & jiggle-on-toes & jumping-jacks & chicken-wings \\
            \midrule
N-ICP & 3.70 / 8.54 & 1.73  / 12.21 & 10.16 / 14.18 & 4.74 / 14.29 \\
AMM & \underline{0.97} / {2.18} & \underline{0.93} / 2.60 & \underline{4.50} / {4.51} & \underline{2.49} / 4.32 \\
BCPD & 73.11  / 34.31 & 77.24  / 22.11 & 84.56 / 34.27 & 74.53  / 46.46 \\
BCPD++ & 57.98 / 3.63 & 77.23 / 2.86 & 80.28  / 3.77 & 51.77 / 5.84 \\
GBCPD & 59.75  / 13.71 & 79.39 / 12.32 & 83.90 / 13.66 & 55.22 / 14.10 \\
GBCPD++ & 53.08 / \underline{1.45} & 44.95 / \underline{1.45} & 35.56  / \underline{1.38} & 36.85  / \underline{1.62} \\ \midrule
Ours & \textbf{0.55} / \textbf{0.53} & \textbf{0.92} / \textbf{0.58} & \textbf{3.51} / \textbf{0.64} & \textbf{2.11} / \textbf{0.63} \\
            \bottomrule
		\end{tabular}
	\end{small}
\end{table}

%% file: fig_table_scripts/table_ablas_all.tex
\begin{table*}[t]
	\caption{
		{Mean value of  $\correrr(\times 0.001) \downarrow$ and AUC $\uparrow$ using different variants of our methods on the AMA dataset~\cite{vlasic2008articulated} and the SHREC'20 track non-rigid correspondence dataset~\cite{Dyke2020tracka}.}}
	\label{Tab:ablations}
	\setlength{\tabcolsep}{3.8pt}
	\centering
    \renewcommand{\arraystretch}{0.8}
	\begin{small}
		\begin{tabular}{ c | c c | c c |c |c }
			\toprule
            Variants & $E_{\alignm}^{C}$ / $E_{\alignm}$ & Robust Weights &  handstand & march1 & {SHREC'20~\cite{Dyke2020tracka}} &  Partial data from~\cite{vlasic2008articulated} \\
			\midrule
    With $E_{\pp}$ & $E_{\pp}$  & \checkmark & 2.91 / 0.77 &1.28 / 0.89 &2.81 / 0.74 & 2.42 / 0.77 \\
    With $E_{\ppl}$ & $E_{\ppl}$  & \checkmark &  0.56 / \underline{0.95} & \underline{0.09} / \underline{0.99} &1.76 / 0.83 &0.84 / 0.93  \\\midrule
    Welsch &  \checkmark  & Welsch & 0.44 / \textbf{0.96} &0.10 / 0.99 & \underline{1.53} / \underline{0.85} & \underline{0.58} / \underline{0.95} \\
    Huber  & \checkmark & Huber & \underline{0.41} / \textbf{0.96} & 0.19 / \underline{0.99} &1.88 / 0.84 &0.78 / 0.93 \\
    Geman McClure & \checkmark  & Geman McClure & 0.43 / \textbf{0.96}  & 0.21 / \underline{0.99} &1.90 / 0.84 &1.09 / 0.91 \\
    Without Weight & \checkmark  &  & 0.75 / \underline{0.95} &0.39 / \underline{0.99} &1.89 / 0.84 & 1.35 / 0.89 \\
    Hard Thres. & \checkmark  & Eq.~\eqref{eq:reliability-weight-hard-thres} & 0.63 / \underline{0.95} &0.25 / \underline{0.99} &1.93 / 0.83 & 1.34 / 0.89 
    \\\midrule
    Ours & \checkmark & \checkmark & \textbf{0.38} / \textbf{0.96} & \textbf{0.03} / \textbf{1.00} & \textbf{1.47} / \textbf{0.86} & \textbf{0.46} / \textbf{0.96} \\
			\bottomrule
		\end{tabular}
	\end{small}
\end{table*}

%% file: conclusion.tex
\section{Limitations and Future Work}
Although our method has achieved good performance in many cases, 
it is important to note that our method primarily focuses on local non-rigid registration. By relying on the nearest point in Euclidean space for establishing correspondences, our method is limited by the initial spatial position. As a result, when significant differences exist between the source and target shapes, the incorporation of global information becomes crucial. This can be achieved through the integration of semantic clues, overall shape analysis, other global features, {and initial correspondences predicted by learning-based methods}. By incorporating global information, our method can potentially enhance its performance, robustness, and ability to handle larger deformations. Exploring the integration of global information represents a valuable direction for future improvements to our approach. {Additionally, our method does not currently account for changes in topology during deformation. For example, when a surface deforms and splits into two or more disconnected parts (as shown in Fig.~\ref{fig:deepdeformface}), our current approach is unable to handle such scenarios. A valuable direction for future work will be to develop techniques that can dynamically adjust the deformation graph and local connection relationships in response to these topological changes.}

\section{Conclusion}

We developed a novel optimization-based method for non-rigid surface registration, which offers several key advantages over existing approaches. Our method leveraged a symmetrized point-to-plane distance metric, resulting in a more precise alignment of geometric surfaces. By incorporating adaptive robust weights, our method effectively handles data with defects such as noise, outliers, or partial overlaps. To address the complexity of the objective function, we employed an alternating optimization scheme and designed a surrogate function that is easy to solve. Additionally, we introduced a graph-based deformation technique for coarse alignment to improve both the accuracy and solution speed. Experimental results demonstrated the superiority of our proposed method compared to state-of-the-art techniques in most examples. Our method achieved the best performance while maintaining a relatively fast solution speed.

%% file: supplementary.tex
\appendices

\section{Technical Details}

\subsection{Subproblem of Updating $\widehat{\mathbf{v}}_i$}
\label{appx:matrix-definition}

This problem of updating $\widehat{\mathbf{v}}_i$ can be written in a matrix form:
\begin{equation}
\label{eq:optimize-vi-fine}
\{\newpos{\mathbf{v}}_i^{(k+1)}\} = \argmin_{\{\newpos{\mathbf{v}}_i\}} \frac{1}{|\mathcal{V}|}\|\mathbf{W}\mathbf{N}(\newpos{\mathbf{V}}-\mathbf{T})\|^2 + \frac{w_{\arap}}{2\mathcal{E}}\|\mathbf{B}\newpos{\mathbf{V}}-\mathbf{Y}\|^2,
\end{equation}
where $\newpos{\mathbf{V}}, \mathbf{T} \in \mathbb{R}^{3 |\mathcal{V}|}$
concatenate $\{\newpos{\mathbf{v}}_i\}$ and $\{\mathbf{u}_{\corresidx{i}^{(k+1)}}\}$ respectively. 
$\mathbf{W} =\diag(\sqrt{\pointpairweight{1}^{(k+1)}}, \ldots, \sqrt{\pointpairweight{|\mathcal{V}|}^{(k+1)}}) \in \mathbb{R}^{|\mathcal{V}| \times |\mathcal{V}|}$ stores the robust weights; 
$\mathbf{N} \in \mathbb{R}^{|\mathcal{V}| \times 3 |\mathcal{V}|}$ is a sparse matrix where the $i$-th row has coefficients $(\mathbf{R}_i^{(k)}\mathbf{n}_i + \mathbf{n}_{\corresidx{i}}^t)^T$ at the columns corresponding to $\newpos{\mathbf{v}}_i$;
$\mathbf{B} \in\mathbb{R}^{6|\mathcal{E}|\times 3|\mathcal{V}|}$ is a block sparse matrix that maps the source point positions $\newpos{\mathbf{V}}$ to scaled edge vectors between neighboring points, where the scaled vector from $\newpos{\mathbf{v}}_i$ to point $\newpos{\mathbf{v}}_j$ correspond to a $3\times 3|\mathcal{V}|$ block row with two non-zero blocks $-\mathbf{I}_3/\sqrt{|\mathcal{N}(\mathbf{v}_i)|}$ and $\mathbf{I}_3/\sqrt{|\mathcal{N}(\mathbf{v}_i)|}$ corresponding to $\newpos{\mathbf{v}}_i$ and $\newpos{\mathbf{v}}_j$ respectively, and $\mathbf{I}_3$ is a $3\times 3$ identity matrix; $\mathbf{Y} \in\mathbb{R}^{6|\mathcal{E}|}$ concatenates the vectors $\mathbf{R}_i^{(k)}\cdot(\mathbf{v}_i-\mathbf{v}_j)/\sqrt{|\mathcal{N}(\mathbf{v}_i)|}$ in an order consistent with $\mathbf{N}$. 
Then the problem in Eq.~\eqref{eq:optimize-vi-fine} can be solved via a linear system
\begin{equation}
\begin{aligned}
&(\mathbf{N}^T \mathbf{W}^T \mathbf{W} \mathbf{N}/|\mathcal{V}| + w_{\arap} \mathbf{B}^T \mathbf{B}/2|\mathcal{E}|)
\newpos{\mathbf{V}}^{(k+1)}\\
= & ~\mathbf{N}^T\mathbf{W}^T\mathbf{W} \mathbf{T}/|\mathcal{V}| + w_{\arap} \mathbf{B}^T\mathbf{Y}/2|\mathcal{E}|. 
\end{aligned}
\label{eq:solve-V}
\end{equation}
As the system matrix is sparse, symmetric, and positive definite, we solve it using Cholesky factorization. Moreover, the sparsity pattern of the matrix remains the same in each iteration. Thus, to improve efficiency, we perform symbolic factorization once before the first iteration, and only carry out numerical factorization in each iteration.

\subsection{Proof that $\overline{f}(\mathbf{R}_i|\mathbf{R}_i^{(k)})$ in Eq.~\eqref{eq:f-Ri-surrogate} is a Surrogate Function for $f(\mathbf{R}_i)$}
\label{appx:surrogate-function}
If $\overline{f}(\mathbf{R}_i|\mathbf{R}_i^{(k)})$ in Eq.~\eqref{eq:f-Ri-surrogate} is a surrogate function for $f(\mathbf{R}_i)$, then it needs to satisfied the conditions~\eqref{eq:surrogate-condition}. We consider two cases:
\begin{itemize}[leftmargin=*]
    \item If $\mathbf{d} = \mathbf{0}$, then
    \[
        \overline{f}(\mathbf{R}_i|\mathbf{R}_i^{(k)}) = f(\mathbf{R}_i) = 0, 
    \]
    which obviously meet the conditions~\eqref{eq:surrogate-condition}. 
    \item If $\mathbf{d} \neq \mathbf{0}$, then from Proposition~\ref{prop:f} we have 
    \begin{align*}
    f(\mathbf{R}_i) &= \|\mathbf{d}\|^2\cdot \min_{\mathbf{h}\in\mathcal{P}}\|\mathbf{R}_i\mathbf{n}_i - \mathbf{h}\|^2\\
    & \leq \|\mathbf{d}\|^2\cdot\|\mathbf{R}_i\mathbf{n}_i-\mathbf{h}_*^{(k)}\|^2 
    = \overline{f}(\mathbf{R}_i|\mathbf{R}_i^{(k)}) \quad \forall \mathbf{R}_i,
    \end{align*}
    and 
    \begin{align*}
    f(\mathbf{R}_i^{(k)}) &= \|\mathbf{d}\|^2\cdot \min_{\mathbf{h}\in\mathcal{P}}\|\mathbf{R}_i^{(k)}\mathbf{n}_i - \mathbf{h}\|^2\\
    & = \|\mathbf{d}\|^2\cdot\|\mathbf{R}_i^{(k)} \mathbf{n}_i-\mathbf{h}_*^{(k)}\|^2 
    = \overline{f}(\mathbf{R}_i^{(k)}|\mathbf{R}_i^{(k)}),
    \end{align*}
    which also satisfies the conditions~\eqref{eq:surrogate-condition}.
\end{itemize}

\subsection{{Details and} Numerical Solver for Coarse Alignment}
\label{appx:solver-coarse}

{The smoothness term $E_{\smooth}$ for coarse alignment can be written as 
    \begin{equation}
    E_{\smooth} = \frac{1}{2|\mathcal{E}_{\mathcal{G}}|} \sum_{\mathbf{p}_i\in\mathcal{V}_{\mathcal{G}}} \sum_{\mathbf{p}_j\in\mathcal{N}(\mathbf{p}_i)}\|\mathbf{D}_{ij}\|^2,
    \end{equation}
    where $\mathbf{D}_{ij}$ is the difference of the graph node $\mathbf{p}_i$ deformed by the affine transformations $\mathbf{X}_j$ and $\mathbf{X}_i$ and calculated by
    \[\mathbf{D}_{ij} = r_{ij}[\mathbf{A}_j(\mathbf{p}_i-\mathbf{p}_j)+\mathbf{p}_j + \mathbf{t}_j - (\mathbf{p}_i+\mathbf{t}_i)],
    \]
    and 
    $r_{ij}=\displaystyle\frac{2|\mathcal{E}_{\mathcal{G}}|\cdot \|\mathbf{p}_i-\mathbf{p}_j\|^{-1}}{\sum_{\mathbf{p}_i\in\mathcal{V}_{\mathcal{G}}}\sum_{\mathbf{p}_j\in\mathcal{N}(\mathbf{p}_i)}\|\mathbf{p}_i-\mathbf{p}_j\|^{-1}}
    $
    is a normalization weight.
    $E_{\rot}$ is to require the affine transformation associated with each node to be close to a rigid transformation:
    \begin{equation}
    E_{\rot} = \frac{1}{|\mathcal{V}_{\mathcal{G}}|} \sum_{\mathbf{p}_j\in\mathcal{V}_{\mathcal{G}}}\|\mathbf{A}_j - \proj_{\mathcal{R}}(\mathbf{A}_j)\|_F^2,
    \end{equation}
    where $\proj_{\mathcal{R}}(\cdot)$ is the projection operator and $\mathcal{R}$ is the rotation matrix group.
    Compared to the ARAP term, $E_{\rot}$ enforces local rigidity at a larger scale and helps to maintain the structure of the source shape.
    Further details of these two terms can be found in~\cite{yao2022fast}.}
    
Similar to the numerical solver in Sec.~\ref{Sec:numericalSolver}, we replace position variables $\{\newpos{\mathbf{v}}_i\}$ with transformations $\{(\mathbf{A}_j,\mathbf{t}_j)\}$ associated with graph nodes, and alternately update 
(1) the closest points $\{\mathbf{u}_{\rho_i}\}$,
(2) the transformations $\{(\mathbf{A}_j,\mathbf{t}_j)\}$, and
(3) the rotation matrix variables $\{\mathbf{R}_i\}$. 
Except for (2), others are the same as Sec.~\ref{Sec:numericalSolver}. 
When solving for the transformation $\{(\mathbf{A}_j,\mathbf{t}_j)\}$ while fixing  $\{\mathbf{u}_{\rho_i}\}$ and $\{\mathbf{R}_i\}$, we observe that the objective function is a quadratic function except for the rotation matrix term $E_{\rot}$ because the projection operator $\proj_{\mathcal{R}}(\cdot)$ depends on $\mathbf{A}_j$. Following~\cite{yao2022fast}, we fix $\mathbf{A}_j$ to $\mathbf{A}_j^{(k)}$ and derive a quadratic problem. The matrix form of this quadratic problem can be written as
\begin{equation}
\label{eq:optimize-x-Coarse}
\begin{aligned}
\mathop{\arg\min}_{\{\mathbf{X}_j\}} &\frac{1}{|\mathcal{S}|}\|\mathbf{W}\mathbf{N}(\mathbf{F}\mathbf{X}-\mathbf{Q})\|^2 + \frac{w_{\arap}^{C}}{2|\mathcal{E}|}\|\mathbf{B}\mathbf{X}-\mathbf{L}\|^2 \\
& + \frac{w_{\smooth}}{2|\mathcal{E}_{\mathcal{G}}|}\|(\mathbf{H}\mathbf{X}-\mathbf{Y})\|^2+ \frac{w_{\rot}}{|\mathcal{V}_{\mathcal{G}}|}\|\mathbf{J}\mathbf{X}-\mathbf{Z}\|^2. 
\end{aligned}
\end{equation}
where $\mathbf{X}_j=[\overline{\mathbf{A}}_j^T, \mathbf{t}_j^T]^T\in\mathbb{R}^{12}$. $\overline{\mathbf{A}}_j$ is a vector obtained by stacking the columns of $\mathbf{A}_j$.
$\mathbf{X}=[\mathbf{X}_1^T,...,\mathbf{X}_{|\mathcal{V}_{\mathcal{G}}|}^T]^T\in\mathbb{R}^{12|\mathcal{V}_{\mathcal{G}}|}$ concatenates all $\mathbf{X}_j$. 
$\mathbf{F}\in\mathbb{R}^{3|\mathcal{V}|\times 12|\mathcal{V}_{\mathcal{G}}|}$ is a block matrix with each block 
\[
\mathbf{F}_{ij} = \left\{
\begin{aligned}
&w_{ij}\cdot[\mathbf{v}_{i}-\mathbf{p}_j, 1]^T\otimes \mathbf{I}, & \text{if }\mathbf{p}_j\in\mathcal{I}(\mathbf{v}_{i}) \\
&\mathbf{0}, & \text{otherwise},
\end{aligned}
\right.
\]
where $\otimes$ is the Kronecker product,  $\mathbf{I}_3\in\mathbb{R}^{3\times 3}$ is the identity matrix.  
$\mathbf{W} =\diag(\sqrt{\widetilde{\alpha}_1}, ..., \sqrt{\widetilde{\alpha}_{|\mathcal{V}|}})\in\mathbb{R}^{|\mathcal{V}|\times |\mathcal{V}|}$ stores the robust weights $\{\pointpairweight{i}^{(k+1)}\}$ defined in Eq.~\eqref{eq:robustweight}, where $\widetilde{\alpha}_{i}=\pointpairweight{i}^{(k+1)}$ if $\mathbf{v}_i\in\mathcal{S}$ and $\widetilde{\alpha}_i=0$ for the others. 
$\mathbf{Q}\in\mathbb{R}^{3|\mathcal{V}|\times 1}$ is a block vector with each block 
\[
\mathbf{Q}_i = \mathbf{u}_{\corresidx{i}^{(k+1)}}-\sum_{\mathbf{p}_j\in\mathcal{I}(\mathbf{v}_{i})}w_{ij}\mathbf{p}_j, 
\]
$\mathbf{N}\in\mathbb{R}^{|\mathcal{V}|\times 3\mathcal{V}}$ is a block diagonal matrix with each block $\mathbf{N}_{ii}=(\newpos{\mathbf{n}}_{i}^{(k)}+\mathbf{n}^t_{\corresidx{i}^{(k+1)}})^T$, where $\newpos{\mathbf{n}}_{i}^{(k)}=\mathbf{R}_i^{(k)}\mathbf{n}_i$ according to Eq.~\eqref{eq:updateNormal}. 

For the ARAP term, $\mathbf{B}\in\mathbb{R}^{6|\mathcal{E}|\times 12|\mathcal{V}_{\mathcal{G}}|}$ is a sparse block matrix, where each row block of is associated with $(\mathbf{v}_{i},\mathbf{v}_j)$ and is equals to $(\mathbf{F}_{i:}-\mathbf{F}_{j:})/\sqrt{|\mathcal{N}(\mathbf{v}_{i})|}$, where $\mathbf{F}_{i:}$ and $\mathbf{F}_{j:}$ is the $i$-th and $j$-th row block of $\mathbf{F}$ respectively. 
Each row of $\mathbf{L}\in\mathbb{R}^{6|\mathcal{E}|}$ is associated with $(\mathbf{v}_{i},\mathbf{v}_j)$ with the values $\sum_{\mathbf{v}_j\in\mathcal{N}(\mathbf{v}_{i})}\|\mathbf{R}_{i}^{(k)}(\mathbf{v}_{i}-\mathbf{v}_j)\|/\sqrt{|\mathcal{N}(\mathbf{v}_{i})|}$.

For the smoothness term, $\mathbf{H}\in\mathbb{R}^{6|\mathcal{E}_{\mathcal{G}}|\times 12|\mathcal{V}_{\mathcal{G}}|}$ and $\mathbf{Y}\in\mathbb{R}^{6|\mathcal{E}_{\mathcal{G}}|\times 1}$ store the computation of a term $\mathbf{D}_{ij}$ in the same row. For each $\mathbf{D}_{ij}$, $\mathbf{H}$ stores the two non-zero blocks $[r_{ij}(\mathbf{p}_i^T-\mathbf{p}_j^T), r_{ij}]\otimes \mathbf{I}_3$ and $[0,0,0,-r_{ij}]\otimes \mathbf{I}_3$ corresponding to $\mathbf{X}_j$ and $\mathbf{X}_i$ respectively. $\mathbf{Y}$ stores $r_{ij}(\mathbf{p}_i-\mathbf{p}_j)$.

For rotation matrix term, $\mathbf{J}\in\mathbb{R}^{9|\mathcal{V}_{\mathcal{G}}|\times 12|\mathcal{V}_{\mathcal{G}}|}$ is a block diagonal matrix with each block $\mathbf{J}_{ii} = [\mathbf{I}_9, \mathbf{0}]\in\mathbb{R}^{9\times 12}$, where $\mathbf{I}_9\in\mathbb{R}^{9\times 9}$ is the identity matrix. $\mathbf{Z}=[\overline{\proj_{\mathcal{R}}(\mathbf{A}_1^{(k)})}^T,...,\overline{\proj_{\mathcal{R}}(\mathbf{A}^{(k)}_{|\mathcal{V}_{\mathcal{G}}|})}^T]^T\in\mathbf{R}^{9|\mathcal{V}_{\mathcal{G}}|}$, where $\overline{\proj_{\mathcal{R}}(\mathbf{A}_i^{(k)})}$ is the vector stores $\proj_{\mathcal{R}}(\mathbf{A}_i^{(k)})$ by column pivot.

Therefore, the optimal solution $\mathbf{X}^{(k+1)}$ can be solved by a system of equations 
\begin{equation}
\label{eq:solve-x}
\mathbf{K}^{(k)}\mathbf{X}^{(k+1)} = \mathbf{b}^{(k)}.
\end{equation}
where $\mathbf{K}^{(k)}$ and $\mathbf{b}^{(k)}$ are respectively: 
\[
\begin{aligned}
& \frac{1}{|\mathcal{S}|}(\mathbf{WNF})^T\mathbf{WNF} + \frac{w_{\arap}^C}{2|\mathcal{E}|}\mathbf{B}^T\mathbf{B}+ \frac{w_{\smooth}}{2|\mathcal{E}_{\mathcal{G}}|}\mathbf{H}^T\mathbf{H} + \frac{w_{\rot}}{|\mathcal{V}_{\mathcal{G}}|}\mathbf{J}^T\mathbf{J},\\
& \frac{1}{|\mathcal{S}|}(\mathbf{WNF})^T\mathbf{WNQ} + \frac{w_{\arap}^C}{2|\mathcal{E}|}\mathbf{B}^T\mathbf{L} + \frac{w_{\smooth}}{2|\mathcal{E}_{\mathcal{G}}|}\mathbf{H}^T\mathbf{Y} + \frac{w_{\rot}}{|\mathcal{V}_{\mathcal{G}}|}\mathbf{J}^T\mathbf{Z}.
\end{aligned}
\]

After updating $\{\mathbf{X}_j\}$, $\{\newpos{\mathbf{v}}_i\}$ can be calculated by Eq.~\eqref{eq:deformed-pos}.

\mypara{Termination criteria} 
The termination criteria for coarse alignment are the same as in Sec.~\ref{Sec:numericalSolver}, except that $\epsilon$ is set to $10^{-3}$. Algorithm~\ref{Alg:coarse-registration} illustrates the solver for coarse alignment using a deformation graph.

{
\subsection{Usage of Initial Correspondences}
When we obtain the initial correspondence $\mathcal{L} = \{\mathbf{v}_l, \mathbf{u}_l\}$ through global information or manual labeling, we introduced a landmark term 
\begin{equation}
E_{\text{landmark}} = \sum_{l} \|\widehat{\mathbf{v}}_l - \mathbf{u}_{l}\|^2
\end{equation}
similar to~\cite{Amberg2007}. Then we added $w_{\text{landmark}}E_{\text{landmark}}$ 
to the objective functions in problems \eqref{eq:FineOptimization} and \eqref{eq:coarseOptimization}, where $w_{\text{landmark}}=100/|\mathcal{L}|$ by default.  
}

\begin{algorithm}[t]
	\caption{Coarse alignment using a deformation graph.}
	\label{Alg:coarse-registration}
	\KwIn{$\{\mathbf{v}_i, \mathbf{n}_i\}_{i=1}^{|\mathcal{V}|}$: the source points and normals;\\ ~~~~$\{\mathbf{u}_i, \mathbf{n}_i^t\}_{i=1}^{|\mathcal{U}|}$: the target points and normals;\\ 
    \\~~~~$K$: maximum number of iterations;\\~~~~$\epsilon$: convergence threshold. 
	}
	\KwResult{The deformed point positions $\{\newpos{\mathbf{v}}_i\}_{i=1}^{|\mathcal{V}|}$.
	} 
	\BlankLine
        Initialize deformation graph $\{\mathcal{V}_{\mathcal{G}},\mathcal{E}_{\mathcal{G}}\}$ following~\cite{yao2022fast}\;
	Initialize $\mathbf{X}^{(0)}$ with identity transformations\;
 Set $\mathbf{R}_i^{(0)} = \mathbf{I}$ and $\newpos{\mathbf{v}}_i^{(0)}=\mathbf{v}_i$ for all $i$\;
 Construct a subset $\mathcal{S}$ from $\mathcal{V}$ by the farthest point sampling\;
 $k=0$\;
 \While{$k < K$ and ${\|\newpos{\mathbf{V}}^{(k+1)}-\newpos{\mathbf{V}}^{(k)}\|}/{\sqrt{|\mathcal{V}|}} <\epsilon$}{
  For each $\mathbf{v}_i\in\mathcal{S}$, find the closest point $\mathbf{u}_{\corresidxiter{i}{k+1}}$ for $\newpos{\mathbf{v}}_i^{(k)}$\;
    Compute weight $\pointpairweight{i}^{(k+1)}$ with Eq.~\eqref{eq:robustweight}\;
    Compute $\mathbf{X}^{(k+1)}$ via linear system~\eqref{eq:solve-x}\;
    Compute the deformed positions $\{\newpos{\mathbf{v}}_i^{(k+1)}\}$ by Eq.~\eqref{eq:deformed-pos}\;
    Compute $\{\mathbf{R}_i^{(k+1)}\}$ with Eq.~\eqref{eq:RUpdate}\;
    $\newpos{\mathbf{n}}_i^{(k+1)} = \mathbf{R}_i^{(k+1)}\mathbf{n}_i$\;          
    $k = k+1$\;
}
\end{algorithm}

\section{More ablation studies}

{
\subsection{Performance of Different Numbers of Graph Nodes}
For the coarse alignment stage, the number of graph nodes will affect the accuracy and speed of registration. We set the sampling radius $R=3 \overline{l}_s, 5\overline{l}_s, 10\overline{l}_s, 15\overline{l}_s$ to run our method and show the results in Tab.~\ref{Tab:ablas_nnodes} and Fig.~\ref{fig:ablas-nnodes}. To avoid the influence of the fine stage, we only perform coarse alignment in this part. We can see that a smaller sampling radius $R$, indicating a denser deformation graph,  typically leads to more accurate alignment, and it takes a longer time for optimization and less time for graph construction. 
Nonetheless, an excessive number of nodes tends to make the deformation field resemble pointwise deformation, introducing redundancy that complicates problem-solving and elevates the risk of converging towards suboptimal solutions, such as the case of $R=3\overline{l}_s$. 
We set $R=10\overline{l}_s$ as a trade-off.
}

\input{fig_table_scripts/table_ablas_all_2}

\input{fig_table_scripts/fig_diff_nnodes}
\input{fig_table_scripts/table_ablas_nnode}

{
\subsection{Effectiveness of Different Terms and Coarse Alignment}
Furthermore, we conducted experiments to evaluate the performance of the proposed method when only the initial alignment is used (i.e. \textit{Coarse}) or when no coarse alignment is used (i.e. \textit{Fine}). For the coarse stage, we also examined the impact of omitting each regularization term:
\begin{itemize}[leftmargin=*]
\item \textit{Without ARAP}: we utilized only the coarse stage deformation and did not incorporate the $E_{\text{ARAP}}^C$ term. In this case, the calculation of the normal after deformation in Eq.~\eqref{eq:updateNormal} was replaced by being estimated from the current iteration points $\newpos{\mathbf{v}}_i^{(k)}$. 
\item \textit{Without Smo.}: we utilized only the coarse stage deformation and did not incorporate the $E_{\text{smo}}$ term; 
\item \textit{Without Rot.Mat.}: we utilized only the coarse stage deformation and did not incorporate the $E_{\text{rot}}$ term. 
\end{itemize}

\begin{figure}[t]
	\centering
	\includegraphics[width=\columnwidth]{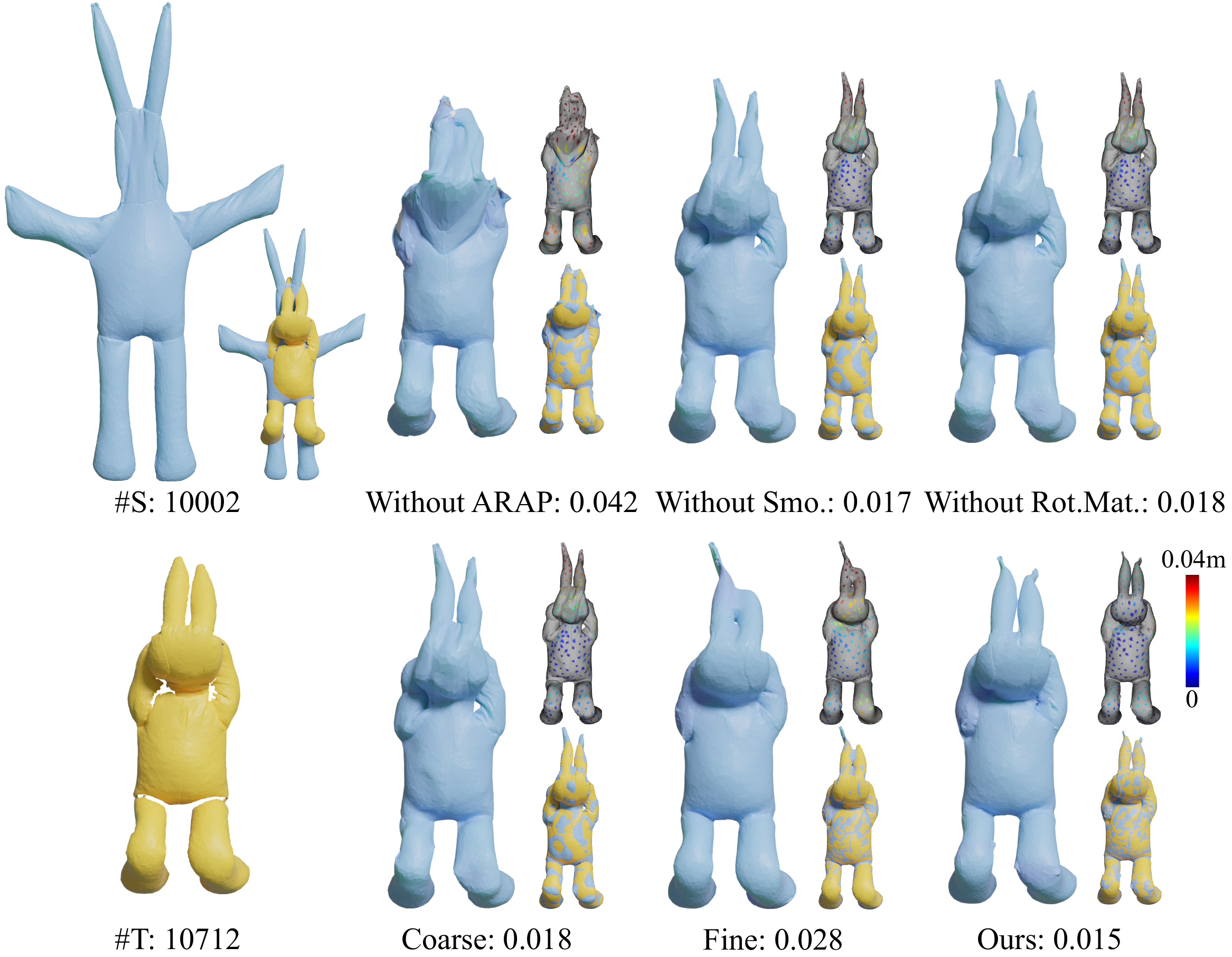}
	\caption{{Comparisons of our method and the variants with ignoring some parts on a problem instance from the SHREC'20 track dataset~\cite{Dyke2020tracka}. For each variant, we show the deformed mesh (left), alignment result (right-bottom), and an error map (right-top) that visualizes the distance between the ground-truth correspondences, as well as label RMSE.}}
	\label{fig:ablas-woterms}
\end{figure}

From Tab.~\ref{Tab:ablations2}, 
we observe that the ARAP constraint played the most crucial role. It ensured the stability of the deformation, while the smoothness term and the rotation matrix term also contributed to maintaining a coherent and well-behaved deformation. 
For the fine stage, when the ARAP term was not used, the absence of regularization led to deformations concentrating on specific points, resulting in poor registration results. So we omitted its results.   Overall, the coarse stage focuses on matching the basic shape, while the fine stage aims to capture local details. The combination of both stages yields superior results. 
Fig.~\ref{fig:ablas-woterms} presents the rendering results for one of these cases.}

\section{The Settings of the Parameters on Experiments}
\label{appx:para-setting}
In the following, we elaborate on the detailed experimental settings:
\begin{itemize}
    \item For N-ICP, we set the coefficient of the smooth term to 10 for all experiments.
    \item For AMM, we set $k_{\alpha}=100, k_{\beta}=1$ for DeformingThings4D datset~\cite{li20214dcomplete} and articulated mesh animation dataset~\cite{vlasic2008articulated}, $k_{\alpha}=50, k_{\beta}=1$ for SHREC'20 Track dataset~\cite{Dyke2020tracka}, $k_{\alpha}=100, k_{\beta}=10$ for constructed data with partial overlaps from articulated mesh animation dataset~\cite{vlasic2008articulated} and the BEHAVE dataset~\cite{bhatnagar22behave}, $k_{\alpha}=1000, k_{\beta}=100$ for DFAUST~\cite{bogo2017dfaust}, DeepDeform~\cite{bozic2020deepdeform} and face sequence from~\cite{guo2015robust}. The sampling radius $R=8\cdot\overline{l}_s$ for constructed data with partial overlaps from~\cite{vlasic2008articulated} and $R=5\cdot\overline{l}_s$ for others.
    \item For BCPD, we set $\omega=0,\beta=0.3,\lambda=10^{-4},\gamma=10,K=150,J=300,c=10^{-6},n=500$ for the SHREC'20 Track dataset~\cite{Dyke2020tracka} and the BEHAVE dataset~\cite{bhatnagar22behave} and $\omega=0,\beta=2.0,\lambda=20,\gamma=10,K=70,J=300,c=10^{-6},n=500$ for others.
    \item For BCPD++, we set $\omega=0,\beta=2,\lambda=50,\gamma=10,J=300,c=10^{-6}$ and $n=500$ for all datasets. $K=150$ for SHREC'20 Track dataset~\cite{Dyke2020tracka} and $K=70$ for others. We set downsampling points 3000 and neighbor ball radius 0.08 for all experiments. 
    \item For GBCPD, we set $\omega=0,\gamma=3,J=300,c=10^{-6}$ and $n=500$ for all datasets. We set $\beta=0.7,\lambda=100,K=200,\tau=0.2$ for DeformingThings4D datset~\cite{li20214dcomplete}, $\beta=1,\lambda=100,K=200,\tau=0.2$ for SHREC'20 Track dataset~\cite{Dyke2020tracka}, and $\beta=1.3,\lambda=500,K=300,\tau=0.5$ for other datasets. 
    \item For GBCPD++, we set $\omega=0,\lambda=100,J=300,c=10^{-6}$ and $n=500$ for all datasets. We set $\beta=0.7,\gamma=1,K=100$ for SHREC'20 Track dataset~\cite{Dyke2020tracka} and DeformingThings4D datset~\cite{li20214dcomplete}, and $\beta=1.2,\gamma=3,K=300$ for others. We set downsampling points 3000 and neighbor ball radius 0.02. 
    \item For LNDP, we use the pre-trained point cloud matching and outlier rejection models with the supervised version. 
    \item For SyNoRiM, we use the pre-trained model ``MPC-DT4D'' to test all datasets. 
    \item For GraphSCNet, we first use the release code of GeoTransformer~\cite{qin2022geometric} with the pre-trained model to obtain the initial correspondences. Then we import these correspondences and the deformation graph nodes obtained by our method into the release code of GraphSCNet, and test all examples with the provided pre-trained model. 
    \item For our method, 
    we set $w_{\smooth}=0.01, w_{\rot}=10^{-4}, w_{\arap}^{C}=500$ and $w_{\arap}=5\times 10^{3}$ for the DeepDeform dataset~\cite{bozic2020deepdeform} and the BEHAVE dataset~\cite{bozic2020deepdeform}, $w_{\smooth}=0.01, w_{\rot}=10^{-4}, w_{\arap}^{C}=5\times 10^3$ and $w_{\arap}=10^{3}$ for face sequence from~\cite{guo2015robust}, $w_{\smooth}=0.001, w_{\rot}=10^{-5}, w_{\arap}^{C}=10$ and $w_{\arap}=200$ for DFAUST dataset~\cite{bogo2017dfaust}, and $w_{\smooth}=0.01, w_{\rot}=10^{-4}, w_{\arap}^{C}=500$ and $w_{\arap}=200$ for other datsets. 
\end{itemize}

%% file: fig_table_scripts/table_ablas_all_2.tex
\begin{table*}[t]
	\caption{
		{Mean value of  $\correrr(\times 0.001) \downarrow$ and AUC $\uparrow$ using different variants of our methods on the AMA dataset~\cite{vlasic2008articulated} and the SHREC'20 track non-rigid correspondence dataset~\cite{Dyke2020tracka}.}}
	\label{Tab:ablations2}
	\setlength{\tabcolsep}{3.8pt}
	\centering
	\begin{small}
		\begin{tabular}{ c | c  c  c  c  c c c | c c |c |c }
			\toprule
            \multirow{2}{*}{\makecell[c]{Variants}} & \multicolumn{4}{c}{Coarse stage} & \multicolumn{2}{c}{Fine stage} & \multirow{2}{*}{\makecell[c]{Robust \\ Weights} } &  \multicolumn{2}{c|}{AMA dataset~\cite{vlasic2008articulated}} & \multirow{2}{*}{SHREC'20~\cite{Dyke2020tracka}} &  \multirow{2}{*}{\makecell[c]{Partial data \\ from~\cite{vlasic2008articulated}}}\\\cmidrule(r){2-5}\cmidrule(r){6-7}\cmidrule(r){9-10} 
			& $E_{\alignm}^C$ & $E_{\arap}^C$ & $E_{\smooth}$ & $E_{\rot}$ & $E_{\alignm}$ & $E_{\arap}$ &  &  handstand & march1  &  \\
			\midrule
    Without ARAP & \checkmark &  & \checkmark & \checkmark &  & & \checkmark &  1.38 / 0.89 &0.96 / 0.93 &2.70 / 0.75 &0.82 / 0.93 \\
    Without Smo. & \checkmark & \checkmark &  & \checkmark &  & & \checkmark & 0.69 / \underline{0.94} & \underline{0.11} / \underline{0.99} & \underline{1.59} / \underline{0.85} &0.51 / \underline{0.95}  \\
    Without Rot.Mat. & \checkmark & \checkmark & \checkmark &  &  & & \checkmark & \underline{0.66} / \underline{0.94} & \underline{0.11} / \underline{0.99} &1.61 / 0.84 & \underline{0.50} / \underline{0.95} \\
    Coarse & \checkmark & \checkmark & \checkmark & \checkmark &  & & \checkmark & 0.69 / 0.93 & 0.12 / \underline{0.99} & \underline{1.59} / \underline{0.85} & 0.51 / \underline{0.95} \\\midrule 
    Fine &  &  &  & & \checkmark &\checkmark & \checkmark & 0.81 / 0.93 &0.41 / 0.97 &1.70 / 0.84 &0.70 / 0.93 \\\midrule
    Ours & \checkmark & \checkmark & \checkmark &\checkmark & \checkmark & \checkmark & \checkmark & \textbf{0.38} / \textbf{0.96} & \textbf{0.03} / \textbf{1.00} & \textbf{1.47} / \textbf{0.86} & \textbf{0.46} / \textbf{0.96} \\
			\bottomrule
		\end{tabular}
	\end{small}
\end{table*}

%% file: fig_table_scripts/fig_diff_nnodes.tex
\begin{figure}[t]
	\centering
	\includegraphics[width=\columnwidth]{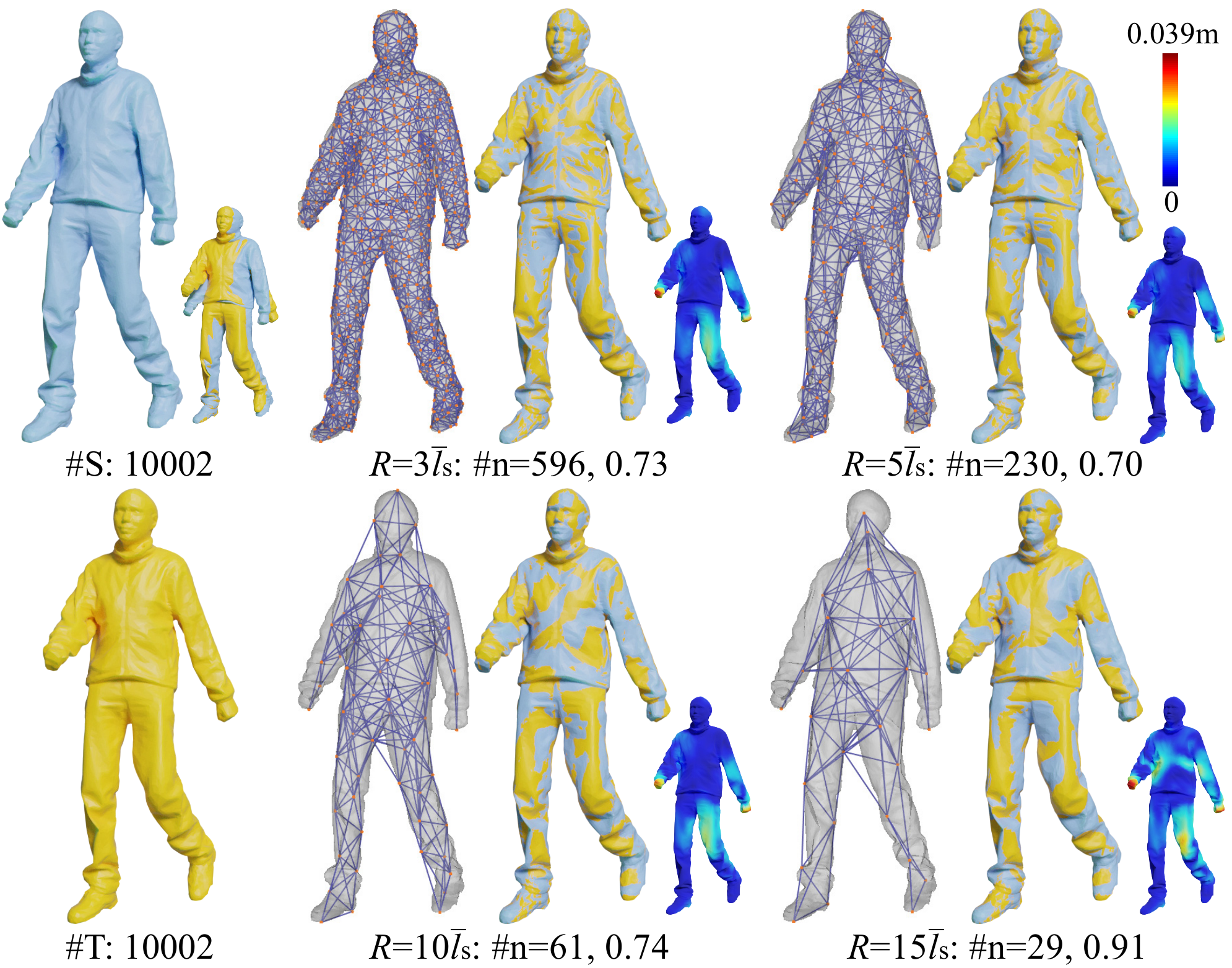}
	\caption{{Comparisons of our method with different number of graph nodes on ``march1'' sequences from the AMA dataset~\cite{vlasic2008articulated}. For each variant, we show the deformation graph (left), alignment result (middle), and an error map (right-bottom) that visualizes the distance between each point and the ground-truth corresponding points on the target shape, as well as label the number of graph nodes and RMSE ($\times 0.01$).}}
	\label{fig:ablas-nnodes}
\end{figure}

%% file: fig_table_scripts/table_ablas_nnode.tex
\begin{table}[t]
	\caption{
		{Average values of RMSE ($\times 0.01$),  $\correrr$ ($\times 0.01$) and running time (s) of our method with different sample radius for deformation graph on ``march1" sequences from~\cite{vlasic2008articulated}. The running time shows the initialization time (s) ( deformation graph construction time / the optimization time)}}
	\label{Tab:ablas_nnodes}
	\setlength{\tabcolsep}{5pt}
	\centering
	\begin{small}
		\begin{tabular}{ c | c c c c}
			\toprule 
            Variants &  RMSE $\downarrow$ & $\correrr$ $\downarrow$ & $\#$nodes &  Time \\
			\midrule
$R = 3\overline{l}_s$  & 0.51 & 0.15 & 587 & 0.50 / 0.79 \\
$R = 5\overline{l}_s$  & 0.40 & 0.11 & 277 & 0.58 / 0.57 \\
$R = 10\overline{l}_s$ & 0.41 & 0.12 & 59 & 0.64 / 0.50 \\
$R = 15\overline{l}_s$ & 0.48 & 0.16 & 27 & 0.68 / 0.46 \\
			\bottomrule
		\end{tabular}
	\end{small}
\end{table}